%% file: main.tex
\documentclass[sigconf, nonacm]{acmart}
\AtBeginDocument{%
  }

\input{preambule/preamb}

\begin{document}

\title{QuaRK: A Quantum Reservoir Kernel for Time Series Learning}

\author{Abdallah Aaraba}
\email{abdallah.aaraba@usherbrooke.ca}
\affiliation{%
  \institution{Université de Sherbrook}
  \city{Sherbrooke}
  \state{Québec}
  \country{Canada}
}

\author{Soumaya Cherkaoui}
\affiliation{%
  \institution{Polytechnique Montréal}
  \city{Montréal}
  \state{Québec}
  \country{Canada}
}

\author{Ola Ahmad}
\affiliation{%
  \institution{Thales cortAIx Labs}
  \city{Montréal}
  \state{Québec}
  \country{Canada}
}

\author{Shengrui Wang}
\affiliation{%
  \institution{Université de Sherbrook}
  \city{Sherbrooke}
  \state{Québec}
  \country{Canada}
}


\begin{abstract}
  Quantum reservoir computing offers a promising route for time series learning by modelling sequential data via rich quantum dynamics while the only training required happens at the level of a lightweight classical readout. However, studies featuring efficient and implementable quantum reservoir architectures along with model learning guarantees remain scarce in the literature. To close this gap, we introduce \textsc{QuaRK}, an end-to-end framework that couples a hardware-realistic quantum reservoir featurizer with a kernel-based readout scheme. Given a sequence of sample points, the reservoir injects the points one after the other to yield a compact feature vector from efficiently measured $k$-local observables using classical shadow tomography, after which a classical kernel-based readout learns the target mapping with explicit regularization and fast optimization. The resulting pipeline exposes clear computational knobs—circuit width and depth as well as the measurement budget—while preserving the flexibility of kernel methods to model nonlinear temporal functionals and being scalable to high-dimensional data. We further provide learning-theoretic generalization guarantees for dependent temporal data, linking design and resource choices to finite-sample performance, thereby offering principled guidance for building reliable temporal learners. Empirical experiments validate \textsc{QuaRK} and illustrate the predicted interpolation and generalization behaviours on synthetic $\beta$-mixing time series tasks.
\end{abstract}




\maketitle

\input{sections/introduction}
\input{sections/background-problem}
\input{sections/method}
\input{sections/generalization}
\input{sections/numerical-validation}

\input{sections/conclusion}

\bibliographystyle{ACM-Reference-Format}
\bibliography{references/refs_main,references/refs_intro,references/refs_apps}

\appendix
\input{appendices/analysis}
\input{appendices/additional_exps}
\end{document}

%% file: preambule/preamb.tex
\usepackage{microtype}
\usepackage{graphicx}
\usepackage{subfigure}
\usepackage{booktabs} 

\usepackage{xcolor}
\usepackage{hyperref}

\usepackage{amsmath}
\usepackage{mathtools}
\usepackage{amsthm}
\usepackage{cleveref}
\crefname{equation}{Eq.}{Eqs.}

\usepackage[textsize=tiny]{todonotes}

\usepackage{amsfonts}
\usepackage[mathscr]{eucal}
\usepackage{cancel}
\usepackage{accents}
\usepackage{enumitem}
\usepackage{braket}
\usepackage{comment}
\usepackage{bm}

\usepackage{float}

\usepackage{empheq}
\usepackage{etoolbox}
\usepackage{xcolor}
\usepackage[normalem]{ulem} 

\usepackage{tikz}
\usetikzlibrary{fit,backgrounds}
\usepackage{quantikz}

\input{preambule/math_cmds}

\input{preambule/therems}

\input{preambule/additional_math}

%% file: preambule/math_cmds.tex
\def\eqref#1{equation~\ref{#1}}

\def\1{\bm{1}}

\def\rvw{{\mathbf{w}}}
\def\rvx{{\mathbf{x}}}

\def\rmW{{\mathbf{W}}}
\def\rmX{{\mathbf{X}}}

\def\vx{{\bm{x}}}
\def\vy{{\bm{y}}}

\DeclareMathAlphabet{\mathsfit}{\encodingdefault}{\sfdefault}{m}{sl}
\SetMathAlphabet{\mathsfit}{bold}{\encodingdefault}{\sfdefault}{bx}{n}

\def\gA{{\mathcal{A}}}
\def\gB{{\mathcal{B}}}
\def\gC{{\mathcal{C}}}
\def\gD{{\mathcal{D}}}
\def\gE{{\mathcal{E}}}
\def\gF{{\mathcal{F}}}

\def\gH{{\mathcal{H}}}
\def\gI{{\mathcal{I}}}
\def\gJ{{\mathcal{J}}}

\def\gN{{\mathcal{N}}}
\def\gO{{\mathcal{O}}}
\def\gP{{\mathcal{P}}}

\def\gS{{\mathcal{S}}}

\def\gU{{\mathcal{U}}}
\def\gV{{\mathcal{V}}}
\def\gW{{\mathcal{W}}}

\def\gY{{\mathcal{Y}}}

\def\sC{{\mathbb{C}}}

\def\sN{{\mathbb{N}}}

\def\sP{{\mathbb{P}}}

\def\sR{{\mathbb{R}}}

\def\sZ{{\mathbb{Z}}}

\newcommand{\E}{\mathbb{E}}

\newcommand{\R}{\mathbb{R}}

\newcommand{\reg}{\lambda}

\DeclareMathOperator*{\argmin}{arg\,min}

\DeclareMathOperator{\Tr}{Tr}

%% file: preambule/therems.tex
\theoremstyle{plain}
\newtheorem{theorem}{Theorem}
\newtheorem{proposition}{Proposition}

\newtheorem{claim}{Claim}
\newtheorem{corollary}{Corollary}
\theoremstyle{definition}

\theoremstyle{remark}

\renewcommand{\qedsymbol}{\rule{2.5mm}{2.5mm}}

%% file: preambule/additional_math.tex
\newcommand\pth[1]{\left(#1\right)}
\newcommand\cro[1]{\left[#1\right]}
\newcommand\acc[1]{\left\{#1\right\}}

\newcommand\norm[1]{\left\lVert#1\right\rVert}
\newcommand\abs[1]{\left| #1 \right|}
\newcommand\normeuc[1]{\norm{#1}_{2}}

\newcommand\normkp[1]{\norm{#1}_{\scriptscriptstyle \kappa}}

\DeclareFontFamily{U}{matha}{\hyphenchar\font45}
\DeclareFontShape{U}{matha}{m}{n}{
<-6> matha5 <6-7> matha6 <7-8> matha7
<8-9> matha8 <9-10> matha9
<10-12> matha10 <12-> matha12
}{}
\DeclareSymbolFont{matha}{U}{matha}{m}{n}

\DeclareFontFamily{U}{mathx}{\hyphenchar\font45}
\DeclareFontShape{U}{mathx}{m}{n}{
<-6> mathx5 <6-7> mathx6 <7-8> mathx7
<8-9> mathx8 <9-10> mathx9
<10-12> mathx10 <12-> mathx12
}{}
\DeclareSymbolFont{mathx}{U}{mathx}{m}{n}

\DeclareMathDelimiter{\vvvert} {0}{matha}{"7E}{mathx}{"17}%

\newcommand\normop[1]{\left \vvvert #1 \right \vvvert_{\infty}}

\newcommand\Trc[1]{\Tr\cro{#1}}

\newcommand{\gSH}{\gS(\gH)}
\newcommand{\gBH}{\gB(\gH)}

\newcommand{\gHkp}{\gH_{\kappa}}

\newcommand\ketbra[1]{\ket{#1}\!\bra{#1}}

\newcommand{\pdensity}{\ketbra{+}^{\otimes n}}

\newcommand\normHS[1]{\norm{#1}_{\scriptscriptstyle \rm HS}}
\newcommand\normTr[1]{\norm{#1}_{1}}

\newcommand{\Zminus}{\sZ_-}
\newcommand{\gIZ}{(\gI)^{\Zminus}}

\newcommand{\longto}{\longrightarrow}

\newcommand\gJarr[1]{\overleftarrow{\gJ}^{\scriptscriptstyle( #1 )}}

\newcommand{\vepsilon}{\boldsymbol{\epsilon}}

\newcommand{\Unif}{\operatorname{Unif}}

\newcommand{\VARMA}{\operatorname{VARMA}}

\newcommand{\thetal}{\theta_\lambda}
\newcommand{\valpha}{\bm{\alpha}}
\newcommand{\pr}{\scriptscriptstyle \operatorname{pr}}

\newcommand{\rmIO}{\mathbf{IO}}
\newcommand{\Reb}{\widehat{\mathfrak{R}}}

%% file: sections/introduction.tex
\section{Introduction}\label{sec:intro}
Time series learning is a cornerstone in numerous modern data mining pipelines where related tasks involve forecasting, classification, anomaly detection, and decision-making in dynamical environments~\cite{esling2012timeseries, chandola2009anomaly, aaraba2024fr}. The main challenge is not only fitting nonlinear temporal dependencies, but also doing so under practical limitations such as limited labeled data, computational budgets, and the need for stable representations when the underlying process exhibits dependence and nonstationary effects. In this direction, reservoir computing~\cite{Jaeger2001EchoState, Grigoryeva_Ortega_2018} has been an attractive modeling approach as it delegates most of the representational burden to a rich dynamical system, while restricting training to a lightweight readout, which yields a favourable trade-off between expressivity and optimization cost.

Quantum reservoir computing (QRC) extends this paradigm by leveraging quantum dynamics to generate features for sequential inputs~\cite{Fujii_Nakajima_2017,Chen_Nurdin_2019,Chen_Nurdin_Yamamoto_2020}. 
By injecting temporal data into a quantum system and subsequently learning a classical readout, QRC proposes a principled way of representing time series with potentially rich, compact internal quantum states. However, there are two obstacles that limit the adoption of this approach as a reliable data mining tool. The first concerns the fact that existing designs are either inconvenient to implement faithfully on hardware or rely on expensive measurement schemes~\cite{mujal2023weakproj, zhu2025minimalistic}. The second is that works on the subject often lack end-to-end frameworks that connect concrete, efficiently implementable QRC architectures to explicit learning-theoretic guarantees on the generalization capacity of such methods~\cite{mujal2021opportunities, qrc_risk_bounds_2025}.

To reduce this gap, we propose in this work an implementable QRC architecture called \textsc{QuaRK} that is systematically designed with scalability in mind. First, a Johnson--Lindenstrauss (JL)~\cite{mohri2018foundations} projection decouples the quantum resources (quantum circuit width and depth) from the original data dimension, enabling the model to remain practical even in the case of high-dimensional inputs. Each time series point is projected to this dimension, yielding a different representation of the initial series with dimensionality matching the number of qubits of the chosen quantum system. Second, the projected coordinates are sequentially injected into the quantum system through simple parameterized rotations followed by a fixed, hardware-friendly entangling layer. Third, using the classical shadows strategy~\cite{huang-classical-shadows-2020}, the quantum reservoir is probed by simultaneously estimating expectation values of $k$-local observables. Finally, following a projected quantum kernel~\cite{Huang_Broughton_Mohseni_Babbush_Boixo_Neven_McClean_2021} philosophy for time series~\cite{Aaraba_Cherkaoui_Ahmad_Laprade_Nahman-Lévesque_Vieloszynski_Wang_2024}, the reservoir readout is learned within the reproducing kernel Hilbert space (RKHS) of a classical kernel function applied to the space of features generated by the quantum featurizer. 

Model capacity is explicitly characterized in a probably approximately correct (PAC) fashion~\cite{ShalevShwartz2014,mohri2018foundations}. In particular, we show that with a number of qubits lower-bounded by the data dimension (number of samples and desired time series length), one can perfectly learn a dataset constructed out of window subsequences of the studied time series. We follow this with an analysis of the generalization capacity of the model conducted on weakly dependent data. In this context, we show that the model exhibits good generalization on unseen data as we increase the number of dataset samples. This learning-theoretic analysis is further validated empirically, showing that the theory matches the empirical results. In summary, our main contributions are as follows:
\begin{enumerate}
    \item We introduce an end-to-end quantum reservoir kernel learner (\textsc{QuaRK}) that couples a hardware-realistic quantum reservoir featurizer with an RKHS kernel readout for temporal learning, while decoupling quantum resources from the data dimension to enable scalability.
    \item We adopt an RKHS-based readout (kernel ridge regression-style) that enables fast closed-form training and interpretable regularization control, applied to a feature space generated via efficient measurement of the quantum reservoir featurizer using classical shadows.
    \item We provide learning-theoretic guarantees for dependent temporal data that relate resource/design choices (projection dimension, reservoir size, multiplexing, measurement budget, and regularization) to finite-sample performance. This is further validated empirically on representative temporal learning tasks.
\end{enumerate}


In what follows, Section~\ref{sec:framework} introduces the learning setup and background on temporal processes and quantum reservoirs, Section~\ref{sec:methodology} presents the \textsc{QuaRK} architecture and theoretical analysis, Section~\ref{sec:experiments} reports empirical validation on $\beta$-mixing time series tasks, and Section~\ref{sec:conslusion} discusses conclusions and future directions.

%% file: sections/background-problem.tex
\section{Framework and background.}\label{sec:framework}
\subsection{Notation}

We denote by $\sZ_- := \{\dots, -2, -1, 0\}$ the set of non-positive integers, by $\sN$ the non-negative integers, and for $m\in\sN^\ast$ we write $[m]:=\{1,\dots,m\}$; for any set $\gA$, $|\gA|$ denotes its cardinality. Random objects are written in uppercase (e.g.\ $X_t,Y_t$) and processes in bold roman (e.g.\ $\rmX=(X_t:t\in\sZ_-)$, $\rmIO=((X_t,Y_t):t\in\sZ_-)$), while a (finite) observed trajectory is written in bold lowercase, $\rvx=(x_t:t\in\sZ_-)\in\gIZ$ with $x_t\in\gI\subset\sR^d$ and $y_t\in\gY\subset\sR$; a length-$w$ window ending at time $\tau$ is $X_\tau^w:=(X_{\tau-w+1},\dots,X_\tau)\in(\gI)^w$ with realization $\rvx_\tau^w=(x_{\tau-w+1},\dots,x_\tau)$, and a windows dataset is $\gD=\{(\rvx_i,y_i)\}_{i=1}^N$ (we use $\E_P[\cdot]$ and $\Pr(\cdot)$ for expectation/probability under the law $P$). For an $n$-qubit register, $\gH:=\sC^{2^n}$, $\gSH$ is the space of density operators, and $\gBH$ is the space of bounded operators on $\gH$; we use Dirac notation, set $\ketbra{+}^{\otimes n} := \otimes_{i=1}^n \ketbra{+}$, and write $\braket{O}_{\rho}:=\Trc{O\rho}$ for the expectation of an observable $O$. We write $\normeuc{\cdot}$ for the Euclidean norm, and $\normHS{\cdot}$, $\normTr{\cdot}$, and $\normop{\cdot}$ for the Hilbert--Schmidt, trace, and operator norms, respectively. For spatial multiplexing with $R\ge1$ sub-reservoirs we work on $\gH^{\otimes R}:=\bigotimes_{r=1}^R\gH$ and $\gBH^R:=\gB(\gH^{\otimes R})$; with a mild abuse of notation, $\gSH^R$ denotes the corresponding state space and, in particular, the product states $\bigotimes_{r=1}^R\rho^{(r)}$ with each $\rho^{(r)}\in\gSH$. Finally, $\gO$ denotes the set of measured $k$-local observables, $m:\gSH^R\to\sR^{R|\gO|}$ the moment map producing features $\Phi=m\circ H^T$, and $\kappa$ is the Mat\'ern kernel with RKHS $\gHkp$ and norm $\normkp{\cdot}$; $K$ denotes the associated Gram matrix and $\lambda_{\rm reg}$ the Tikhonov regularization parameter.


\subsection{Learning temporal data processes.}\label{subsec:learning-temporal-data}
Let 
$\rmIO = \pth{(X_t, Y_t): t \in \sZ_-}$, where $ X_t \in \gI$ and $Y_t \in \gY$, be a semi-infinite input/output stochastic process\footnote{The probability space $(\Omega, \gA, \sP)$ is fixed for all random variables and processes.}, with $\gI \subset \sR^{d}$ being some bounded measurable input set such that $\sup_{x \in \gI}\normeuc{x} \leq \Upsilon_{\scriptscriptstyle X}$, and $\gY \subset \sR$ some bounded measurable real-valued output set satisfying $\sup_{y \in \gY} |y| \leq \Upsilon_{\scriptscriptstyle Y}$.
We distinguish between the stochastic process and its realizations: the input/output process $\mathrm{IO}$ is a random object defined on a fixed probability space, while a (finite) time series corresponds to a single observed trajectory of this process, from which we construct windows for learning.
We consider the supervised learning problem where the task is to learn the data generating process (DGP) functional (map) $H^\star$ that assigns $Y_0$ to the process $\rmX = (X_t : t \in \sZ_-)$: $Y_0 = H^\star(\rmX)$, where the process $\rmIO$ is distributed according to some unknown distribution $P$~\cite{Gonon_Grigoryeva_Ortega_2020, qrc_risk_bounds_2025}. This problem is rooted in real-life processes (e.g., weather data) where we only have access to data up to date (with the most recent index being $t=0$) and we'd like to make a prediction $Y_0$ from such historical data (e.g., tomorrow's temperature, after a one-step shift of the time index).

As is central to machine learning, approximating the unknown DGP functional $H^\star$ follows from a risk minimization procedure. The (statistical) risk, or generalization error, associated with a functional $H$ is defined with respect to a fixed $L$-Lipschitz ($L > 0$) loss function $\ell: \sR \times \sR \longto \sR$ (e.g.\ squared error) as
\begin{equation}\label{eq:risk-def}
    R(H) := \E_P[\ell(H(\rmX), Y_0)].
\end{equation}
We are interested in learning $H^\star$ from the family $\gF$ of functionals satisfying the fading memory property (FMP)~\cite{Monzani_Prati_2025}. Intuitively, this property implies that if two time series $\rvx, \rvx' \in \gIZ$ are similar in their recent past (i.e., have similar values $x_t$ and $x'_t$ for times $t > t^\ast$ for some past time point $t^\ast \in \sZ_-$), then their outputs $H^\star(\rvx)$ and $H^\star(\rvx')$ will be close, even if $\rvx$ is very different from $\rvx'$ in the distant past ($t < t^\ast$).

Given a hypothesis class of functionals $\gC \subset \gF$, the ultimate goal of the learning procedure consists in determining the functional $H_\gC$ that exhibits minimal risk by solving
\begin{equation}\label{eq:in-class-optimization}
    H_\gC = \argmin_{H \in \gC} R(H).
\end{equation}
It is worth mentioning that the broader the class $\gC$ is, the more accurate the learning can be, hence the need for a rich (expressive) hypothesis class. Since the risk $R(H)$ depends on an unknown distribution $P$ as well as on a semi-finite time series process $\rmX$, it is generally infeasible to solve the optimization in~\cref{eq:in-class-optimization}. To this end, one must come up with an empirical risk $\hat{R}$ that does not deviate significantly from the true risk\footnote{Having a controllable $|R(H) - \hat{R}(H)|$ is what we mean by a good generalization.}, i.e., a quantity $|R(H) - \hat{R}(H)|$ that is controllable and can be computed efficiently. The learning procedure then becomes minimization of the empirical risk, i.e.\ $\argmin_{H \in \gC} \hat{R}(H)$.

\subsection{Quantum reservoir computers.}\label{subsec:qrc-background}

Crucial to reservoir computers (RCs)~\cite{Jaeger_2001_EchoState_GMD148} is the state-space evolution described as 
\begin{equation}\label{eq:RC-evolution}
    \rho_t = T(\rho_{t-1}, x_t), \quad \text{for all} \; t \in \sZ_-,
\end{equation}
where $T: \gS \times \gI \longto \gS$ is called the evolution map, and $\gS \subset \gB$ is the state space of some normed space $(\gB, \norm{\cdot})$. Quantum reservoir computers (QRCs) are an important class of RCs that leverage the state spaces of quantum systems~\cite{Chen_Nurdin_Yamamoto_2020}. Indeed, QRCs have as state space the set of quantum states described by the convex subset
\begin{equation}
    \gSH = \Set{\rho \in \gBH : \Trc{\rho} = 1, \rho \succeq 0},
\end{equation}
where $\gBH$ is the space of bounded linear operators on the Hilbert space $\gH = \sC^{2^n}$ of an $n$-qubit quantum system. Unless otherwise stated, we mainly consider the Hilbert--Schmidt (HS) norm (i.e., the Schatten-2 norm) $\normHS{\Psi} = \sqrt{\Trc{\Psi^\dag \Psi}}, \Psi \in \gBH$ as the norm on $\gBH$. Furthermore, at each time step, the state of a QRC is represented by a quantum state, while the quantum dynamics of the reservoir are described by a time-evolution map on the same quantum system (see \cite{Hayashi_Ishizaka_Kawachi_Kimura_Ogawa_2015} for details). Time-evolution maps on a quantum system are linear maps that are completely positive and trace-preserving (CPTP). For each fixed input $x_t \in \sR^{d}$, the induced map $T(\cdot, x_t): \gSH \longto \gSH$ is CPTP, and the state evolution is given by $\rho_t = T(\rho_{t-1}, x_t)$.

A QRC induces an inner filter, an inner-reservoir functional, and an outer-reservoir functional, defined respectively as
\begin{equation} \begin{aligned} 
    &U^T: \gIZ \longto {(\gSH)}^{Z_-}, \quad \;\; \rvx \longmapsto \big(T(s_{t-1}, x_{t}): t \in \sZ_- \big)\\ 
    &H^T: \gIZ \longto (\gSH), \qquad \;\; \rvx \longmapsto U^T(\rvx)_0\\ &H_h^T: \gIZ \longto \sR, \qquad\quad\quad\quad \; \rvx \longmapsto h \circ H^T(\rvx), 
\end{aligned} 
\end{equation}
where $h: \gSH \to \sR$ is some readout function that combines a measurement scheme conducted on the quantum system with some classical post-processing.

There are two properties that are crucial in reservoir computing: the echo state property (ESP) and the fading memory property (FMP). The first, ESP, is satisfied when for every sequence $\rvx \in \gIZ$ there exists a sequence of states $(\rho_t: t \in \sZ_-)$ that satisfies the relation in~\cref{eq:RC-evolution} for each $t \in \sZ_-$, and such a solution is unique. The ESP has attracted a lot of attention in the literature on reservoir computing~\cite{Jaeger_2001_EchoState_GMD148} because it gives the system a function-like definition where, for each input sequence, there is a single output sequence. The second property, on the other hand, states that a reservoir computer satisfies the FMP when its functional $H^T_h$ satisfies it in the sense described in the second paragraph of Section~\ref{subsec:learning-temporal-data}. Combined together, these properties imply that a state $\rho_t$ at time $t$ depends less and less on earlier states $\rho_{t-w}$ as we increase $w$. We will leverage this to show that the reservoir's initial state becomes increasingly irrelevant to the output as the reservoir processes more and more time points.

An interesting family of reservoir computers is that of evolution maps $T$ that are strictly contractive in their first argument:
\begin{equation}\label{eq:contraction}
    \norm{T(\rho, x) - T(\rho', x)} \leq \lambda \norm{\rho - \rho'}, \; \text{for all} \; \rho, \rho' \in \gS,
\end{equation}
where $\lambda \in (0,1)$. This is of interest because any reservoir of this form automatically satisfies both the ESP and the FMP~\cite{Monzani_Prati_2025, Grigoryeva_Ortega_2018}. Subsequently, in Section~\ref{subsec:QR-embedding}, we propose a design of quantum reservoirs satisfying this condition, which leads to a principled hypothesis class $\gC_R$ of functionals that we leverage to approximate an unknown DGP functional $H^\star \in \gF$.

%% file: sections/method.tex
\section{Methodology}\label{sec:methodology}
In this section, we present the end-to-end \textsc{QuaRK} pipeline. In Section~\ref{subsec:QR-embedding}, we introduce our quantum reservoir embedding pipeline. Section~\ref{subsec:K-read} outlines how the reservoir is read out efficiently to form feature vectors and how these are lifted into an RKHS for learning. Section~\ref{subsec:readout-training} describes how the resulting kernel-based readout is classically trained. Finally, Section~\ref{subsec:learning-theoretic-analysis} analyzes, in a PAC-style fashion, the link between design/resource choices and finite-sample performance on weakly dependent data.

\subsection{Quantum reservoir embedding}\label{subsec:QR-embedding}
Although we described the machine learning problem in Section~\ref{subsec:learning-temporal-data} as that concerning learning the mapping $H^\star: \rmX \mapsto Y_0$, in reality we do not have a dataset of time series samples $\rvx \in \gIZ$ to train on. All we have is a single time series, i.e.\ a single realization of the process $\rmX$, which we must leverage to approximate the unknown $H^\star$. To this end, we adopt the windows-based learning approach discussed in~\cite{Gonon_Grigoryeva_Ortega_2020} to form an (artificial) windows dataset $\gD = \{(\rvx_1, y_1), \dots, (\rvx_N, y_N)\}$, where each window is a lookback sequence $\rvx$ of the form $\rvx = (x_{\tau-w+1}, \dots, x_{\tau}) \in (\gI)^w$ along with its label $y = y_\tau$ for some $\tau \in \sZ_-$. We override the definition of the functional $H^T(\cdot)$, which was initially defined to accept semi-infinite inputs, with a definition of $H^T(\rvx)$ for a given window $\rvx$ that starts from an initial state $\rho_{\tau-w}$, evolves the reservoir with respect to inputs $x_t \in \rvx$, and returns the output state $\rho_\tau$ at the end of the evolution.

\subsubsection{Projection}\label{subsubsec:projection} 
Let $n \in \sN$ be the number of qubits chosen by the user and let $\rvx \in \gD$. To make our model parameters independent of the data dimension, which is a key property for allowing scalability with potentially high dimensions $d$, we first project each input $x \in \rvx$ into another representation $z \in \sR^n$ using a Johnson--Lindenstrauss (JL) data projection with a linear projection matrix $\Pi \in \sR^{n \times d}$ so that $z = \Pi x$. We use a random matrix projection $\Pi$ distributed according to the Gaussian JL distribution~\cite{Dasgupta_2002}
\begin{equation}\label{eq:JL-matrix}
    \Pi = \pth{\Pi_{ij}}_{i,j} \in \sR^{n \times d}, \qquad \Pi_{ij} \overset{\mathrm{i.i.d.}}{\sim} \gN\pth{0, \tfrac{1}{n}}.
\end{equation}
Such a projection is deliberately chosen to yield low-distortion embeddings, so that the (Euclidean) distances between all points in the dataset---i.e.\ the set $\gP := \bigcup_{i=1}^N \rvx_i \subset \sR^d$ of cardinality at most $wN$---are nearly preserved. Importantly, one can show that if the number of qubits satisfies $n = \Omega\big(\varepsilon_{\rm pr}^{-2} \log\pth{\tfrac{w N}{\delta}}\big)$ for some error tolerance and failure probability $\varepsilon_{\rm pr}, \delta \in (0,1)$, then for all $u,v \in \gP$, we have
\begin{equation}
    (1-\varepsilon_{\rm pr}) \normeuc{u - v}^2 \leq \normeuc{\Pi u - \Pi v}^2 \leq (1+\varepsilon_{\rm pr}) \normeuc{u - v}^2,
\end{equation}
with probability at least $1-\delta$ by the standard JL lemma (see Theorem 2.1 in~\cite{woodruff2014sketching}). This observation is later leveraged to show effective learning with our reservoir computer (see Section~\ref{subsec:effective-learning}).

\subsubsection{Quantum reservoir architecture}
For the QR architecture, we adopt a design that adheres to the family of quantum reservoirs called contracted-encoding quantum channels (CEQC) as introduced in~\cite{Martínez-Peña_Ortega_2025}. Such a design is interesting for numerous reasons, the main one being its automatic satisfaction of the FMP and ESP. These channels are defined as CPTP maps $T: \gSH \times \gI \to \gSH$ that lead to state-space transformations of the form
\begin{equation}\label{eq:CEQC}
    \rho_t = T(\rho_{t-1}, x_t) := \gE \circ \gJ(\rho_{t-1}, x_t),
\end{equation}
where $\gE: \gSH \to \gSH$ is a strictly contractive map, and $\gJ: \gSH \times \gI \to \gSH$ is a CPTP map that encodes the input information. It is worth mentioning that quantum reservoirs defined by CPTP maps of this form encompass numerous architectures in the QRC literature, such as quantum circuits with noise~\cite{Suzuki_2022, Kubota_2023}, reset-rate channels~\cite{Chen_Nurdin_Yamamoto_2020, Molteni_2023}, mid-circuit measurements~\cite{Yasuda_Suzuki_Kubota_Nakajima_Gao_Zhang_Shimono_Nurdin_Yamamoto_2023, Hu_2024}, and compositions of a dissipative channel with an amplitude encoding map~\cite{Mujal_2023}. The following paragraphs present our design for both channels $\gJ$ and $\gE$.

\paragraph{Unitary evolution.}
After projecting an input $x$ into $z = \Pi x$, we inject each component $z_j$ of $z$ into the circuit using $R_y$ gates parametrized by angles $\theta(z_j) = \pi \cdot \tanh(z_j) \in [-\pi, \pi]$, followed by a global fixed unitary $W$, which realizes the evolution $\gJ(\rho,x) = V(x) \rho V(x)^\dagger$, where
\begin{equation}\label{eq:unitary-design}
    V(x) := W \bigotimes_{j=1}^n R_y\big(\theta(z_j)\big)
\end{equation}
as illustrated in Figure~\ref{fig:unitary-evolution}. Furthermore, given our quantum hardware topology $(V, E)$, where the vertices are our $n$ qubits and $E$ contains available edges between qubits of the form $(i,j)$, we realize the unitary $W$ following the Ising-like, hardware-friendly design
\begin{equation}\label{eq:ising-unitary}
    W = \prod_{j \in V} R_x(\vartheta_{x}^j) \prod_{j \in V} R_z(\vartheta_{z}^j) \prod_{(i,j) \in E} R_{zz}(\vartheta_{zz}^{ij}),
\end{equation}
where all parameters $\vartheta_{\bullet}$ are sampled i.i.d.\ uniformly from $(-\pi, \pi)$.

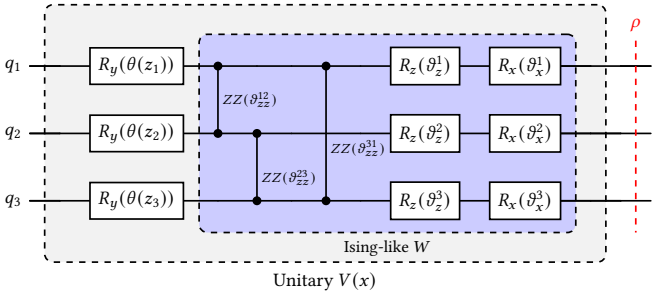
\begin{figure}[t]
\centering
\scalebox{0.8}{
\begin{quantikz}
\lstick{$q_1$} &
\qw \gategroup[wires=3,steps=10,
  style={dashed, rounded corners, fill=gray!10, inner xsep=4pt, inner ysep=17pt},
  background,
  label style={label position=below, anchor=north, yshift=-0.25cm}
]{Unitary $V(x)$} & \gate{R_y(\theta(z_1))} &
  \ctrl[wire style={"ZZ(\vartheta_{zz}^{12})"}]{1}
    \gategroup[wires=3,steps=7,
      style={dashed,rounded corners,fill=blue!20,inner xsep=4pt,inner ysep=3pt},
      background,
      label style={label position=below,anchor=north,yshift=-0.2cm}
    ]{\small Ising-like $W$}
  & & &
  \ctrl[wire style={"ZZ(\vartheta_{zz}^{31})"{right,yshift=-7pt}}]{2} & &
  \gate{R_z(\vartheta^1_z)} &
  \gate{R_x(\vartheta^1_x)} & \qw & \slice[style={red,dashed},label style={red}]{\(\rho\)} &
  \qw \\
\lstick{$q_2$} & \qw &
  \gate{R_y(\theta(z_2))} &
  \control{} &
  \ctrl[wire style={"ZZ(\vartheta_{zz}^{23})"{right,yshift=-5pt}}]{1} & & & &
  \gate{R_z(\vartheta^2_z)} &
  \gate{R_x(\vartheta^2_x)} & \qw & 
  \qw & \\
\lstick{$q_3$} & \qw &
  \gate{R_y(\theta(z_3))} & &
  \control{} & &
  \control{} & &
  \gate{R_z(\vartheta^3_z)} &
  \gate{R_x(\vartheta^3_x)} & \qw &
  \qw &
\end{quantikz}
}
\caption{\small Circuit for the unitary evolution block $V(x)$ on three qubits with ring connectivity.
First, classical features are encoded via single-qubit angle encoding $R_y(\theta(z_j))$ on each $q_j$.
Then the Ising-like unitary $W$ (~\cref{eq:ising-unitary}) is shown as ZZ couplings on the edges $(1,2)$, $(2,3)$, and $(3,1)$,
followed by local rotations $R_z(\vartheta_z)$ and $R_x(\vartheta_x)$.}
\Description{Schematic quantum circuit on three qubits q1--q3 showing an input-dependent unitary block. Each qubit first undergoes a single-qubit Y-rotation Ry(theta(z_j)) to encode classical features. A highlighted inner region then applies three pairwise ZZ couplings arranged in a ring (between qubit pairs 1--2, 2--3, and 3--1). The block ends with local single-qubit rotations Rz and Rx on each qubit. A dashed outer box indicates the overall unitary V(x), and the highlighted region marks the contractive channel component within it.}

\label{fig:unitary-evolution}
\end{figure}

\paragraph{Contractive quantum channel $\gE$.} Let $\lambda \in (0,1)$ be a contraction factor. We design the channel $\gE$ as the reset-rate channel
\begin{equation}\label{eq:contractive-channel}
    \gE_\lambda: \gSH \longto \gSH, \quad \rho \longmapsto \lambda \rho + (1 - \lambda) \pdensity.
\end{equation}
As it is easy to check (with $\gE_\lambda(\rho')$ in the first line below) that
\begin{equation}
\begin{aligned}
    &\normHS{\gE_\lambda(\rho) - \gE_\lambda(\rho)} \leq \lambda \normHS{\rho - \rho'} \\
    \text{and} \;\; &\norm{\gJ(\rho, x) - \gJ(\rho', x)} \leq \norm{\rho - \rho'},
\end{aligned}
\end{equation}
we have $\normHS{T(\rho, x) - T(\rho', x)} \leq \lambda \normHS{\rho - \rho'}$. Hence, we adhere to the condition in the last paragraph of Section~\ref{subsec:qrc-background}, thereby making our reservoir computer satisfy both the ESP and FMP\footnote{Notice also that our map $T = \gE_\lambda \circ \gJ$ is by definition non-unital, i.e., $T(I, x) \neq I$, hence we ensure a non-pathological, non-trivial filter as discussed in Theorem 1 of~\cite{Mart_nez_Pe_a_2023}.}. Figure~\ref{fig:contractive-channel} illustrates how we realize such a reset-rate channel on our quantum circuit, where, controlled on a coin qubit $c$ rotated by $R_y(\thetal)$ with $\thetal = 2 \arcsin \sqrt{1-\lambda}$, we swap reservoir qubits with ancilla ones all initialized to $\ket{+}$. The convex formula in~\cref{eq:contractive-channel} is thus realized, as the coin $c$ becomes superposed with probability $1 - \lambda$ of being in state $\ket{1}$. In what follows, we combine parameters $\vartheta_{\bullet}$ and $\lambda$ into a single vector $\vartheta$.

\paragraph{Quantum embedding of a window.}\label{par:quantum-embedding} Let $\rvx = (x_{\tau-w+1}, \ldots, x_{\tau})$ be a window of our time series. A consequence of using the reset-rate channel $\gE_\lambda$ is that, provided $w$ is large enough, the state at which our quantum reservoir is initialized at time $\tau-w$ matters little to the output state $\rho_\tau$ (see the third paragraph of Section~\ref{subsec:qrc-background}). Hence, by convention, we fix $\rho_{\tau-w} = \pdensity$ as the initial state of the reservoir at time $\tau-w$, and we process the inputs $x_t$ step by step until reaching the output state $\rho_\tau$, realizing the following reservoir evolution:
\begin{equation}
    \rho_{\tau-w} \xlongrightarrow{T(\cdot, \, x_{\tau - w + 1})} \rho_{\tau-w+1} \xlongrightarrow{T(\cdot, \, x_{\tau - w + 2})} \cdots \xlongrightarrow{T(\cdot, \, x_{\tau})} \rho_{\tau}.
\end{equation}
We refer to such a window embedding as $H^T(\rvx) \in \gSH$, which maps the window $\rvx$ to $\rho_\tau$.

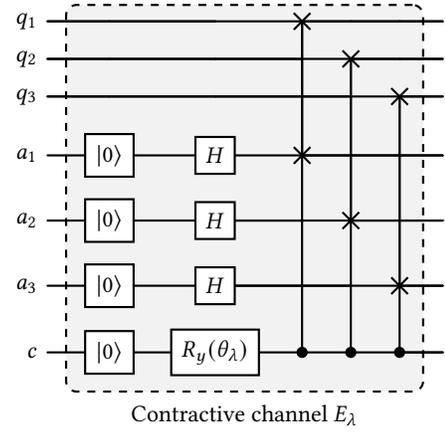
\begin{figure}[t]
\centering
\scalebox{1}{
\begin{quantikz}[row sep=0.3cm]
\lstick{$q_1$} &
\qw
\gategroup[wires=7,steps=5,
  style={dashed, rounded corners, fill=gray!10, inner xsep=4pt, inner ysep=3pt},
  background,
  label style={label position=below, anchor=north, yshift=-0.25cm}
]{Contractive channel $E_\lambda$}
& \qw & \swap{3} & \qw      & \qw      & \qw \\[0.2cm]
\lstick{$q_2$} & \qw & \qw & \qw      & \swap{3} & \qw      & \qw \\[0.2cm]
\lstick{$q_3$} & \qw & \qw & \qw      & \qw      & \swap{3} & \qw \\[0.2cm]
\lstick{$a_1$} & \gate{\ket{0}} & \gate{H} & \swap{-3} & \qw      & \qw      & \qw \\
\lstick{$a_2$} & \gate{\ket{0}} & \gate{H} & \qw      & \swap{-3} & \qw      & \qw \\
\lstick{$a_3$} & \gate{\ket{0}} & \gate{H} & \qw      & \qw      & \swap{-3} & \qw \\
\lstick{$c$}   & \gate{\ket{0}} & \gate{R_y(\theta_\lambda)} & \ctrl{-6} & \ctrl{-5} & \ctrl{-4} & \qw
\end{quantikz}
}

\caption{\small SWAP-dilation realization of the contractive channel $E_\lambda$ acting on a 3-qubit reservoir state $\rho$ (top wires).
The ancilla register (middle wires) is reset to $\ket{0}$ and prepared in $\ket{+}^{\otimes 3}$ via Hadamards.
A coin qubit (bottom wire) is reset and rotated by $R_y(\theta_\lambda)$, where $\theta_\lambda = 2\arcsin\!\sqrt{1-\lambda}$ so that
$\Pr[c=1]=1-\lambda$ and $\Pr[c=0]=\lambda$.
Conditioned on $c=1$, controlled-SWAPs exchange each reservoir qubit $q_i$ with its corresponding ancilla qubit $a_i$.}
\Description{Circuit diagram implementing a contractive channel via a SWAP-dilation. The top three wires are reservoir qubits q1--q3. Three ancilla wires a1--a3 are initialized to |0> and prepared with Hadamard gates. A bottom “coin” qubit c is initialized to |0> and rotated by Ry(theta_lambda). Controlled-SWAP gates, controlled by the coin qubit, swap each reservoir qubit qi with its corresponding ancilla ai when the control is active. A dashed box encloses the full channel implementation.}
\label{fig:contractive-channel}
\end{figure}

\subsubsection{Increased expressivity via spatial multiplexing}
As introduced in~\cite{Nakajima_Fujii_Negoro_Mitarai_Kitagawa_2019} and revised in~\cite{Chen_Nurdin_2019,Chen_Nurdin_Yamamoto_2020}, we adopt the idea of spatial multiplexing to increase model expressivity. The idea is that, instead of using a single reservoir computer, we can use a set of $R \geq 1$ sub-reservoirs with the same architecture but different parameters. These sub-reservoirs are evolved independently (e.g., by running different quantum circuits with different parameters $\vartheta$), so that for a single window $\rvx$ we obtain a set of $R$ different quantum states, thereby increasing the feature vector size and expressivity. We mathematically represent the resulting spatially multiplexed (SM) window embedding using the compact equivalent form 
\begin{equation}
    H^T: \rvx = (x_{\tau-w+1}, \dots, x_{\tau}) \longmapsto \bigotimes_{r=1}^R H^{T_r}(\rvx) := \rho_{\tau}^{(r)} \in \gSH^R,
\end{equation}
with $T_r$ being the evolution map of sub-reservoir $r \in [R]$.
In addition, the parameters $\vartheta_r$ are all sampled independently. Working with different contraction values $\lambda_r$ is beneficial, as it yields different sub-functionals $H^{T_r}$ each with a different fading-memory structure, thereby enriching model expressivity with respect to modeled memories\footnote{It has been shown that the family of QRCs is universal in that of the functionals with the FMP~\cite{Nakajima_Fujii_Negoro_Mitarai_Kitagawa_2019}.}.

\subsection{Reading out the reservoir}\label{subsec:K-read}
In this section, we describe the readout module. We first present an efficient classical shadows measurement scheme to form feature vectors from the reservoir state, then explain how a kernel-based RKHS readout is trained on these features.

\subsubsection{Efficient measurement scheme}\label{subsubsec:efficient-msmt-scheme}
At the end of the reservoir, we read out each sub-reservoir state $\rho_\tau$ by measuring a set of $k$-local observables denoted by $\gO$---in our experiments in Section~\ref{sec:experiments}, we set $\gO$ to be the set of $2$-local Pauli observables\footnote{A Pauli observable $P =\sigma_1 \otimes \cdots \otimes \sigma_n$, with $\sigma_i \in \{I, X, Y, Z\},$ is called $k$-local (or weight $\leq k$) if the index set $\{i \in [n]: \sigma_i \neq I\}$ has cardinality $\leq k$.}. Such a choice is deliberately made so that measuring the sub-reservoirs can be efficiently realized. Indeed, we make use of the classical-shadow estimator~\cite{huang-classical-shadows-2020}, which states that one can simultaneously and efficiently estimate all moments $[\braket{O}_{\rho_\tau}: O \in \gO]$ up to error tolerance $\varepsilon_{\rm cs} \in (0,1)$ using only a number of circuit runs (classical snapshots) scaling as $O\big(\pth{\tfrac{3}{2}}^k \varepsilon_{\rm cs}^{-2} \log(|\gO|)\big)$ by employing random single-qubit Pauli basis measurements~\cite{huang-classical-shadows-2020}.

The moments are estimated using the median-of-means algorithm, as introduced in the seminal paper on classical shadows~\cite{huang-classical-shadows-2020}, which is more resistant to statistical fluctuations. In the end, our measurement scheme for the sub-reservoirs $m: \gSH^R \to \sR^{R |\gO|}$ produces the feature vector
\begin{equation}\label{eq:feature-vector}
    \Phi(\rvx) = m \pth{H^T(\rvx)} := \cro{\braket{O}_{\rho_\tau^{(r)}}: O \in \gO, \;  r \in [R]}^\top \in \sR^{R |\gO|},
\end{equation}
where $H^T(\rvx) = \bigotimes_r \rho_\tau^{(r)}$.
Subsequently, feature vectors of this kind form the inputs to the classical Mat\'ern kernel we use in the following section.

\subsubsection{RKHS readout}
Let $\varphi$ be a (unit-variance) Mat\'ern covariance with smoothness and lengthscale parameters $\nu > 1$ and $\xi > 0$, respectively, defined as
\begin{equation}\label{eq:matern-profile}
    \varphi(s) := \frac{2^{1-\nu}}{\Gamma(\nu)} \pth{\sqrt{2 \nu}\; \frac{s}{\xi}}^\nu K_\nu\pth{\sqrt{2\nu} \; \frac{s}{\xi}},
\end{equation}
where $\Gamma$ is the gamma function, and $K_\nu$ is the modified Bessel function of the second kind~\cite{genton2001classes}. We then define our kernel function $\kappa$ on $\gSH^R$ as
\begin{equation}\label{eq:matern-kernel}
    \kappa: 
    \begin{array}{rcl}
    \gSH^R \times \gSH^R &\longto& \sR\\
    (\rho, \rho') &\longmapsto& 
        \varphi \Big(\normeuc{m(\rho)-m(\rho')} \Big).
    \end{array}
\end{equation}
Finally, the classical readout functions $h$ we use are drawn from the reproducing kernel Hilbert space (RKHS) $\gHkp$ of $\kappa$ defined as the closure (inclusion of limit points) of the set of functions~\citep{Hofmann_Schölkopf_Smola_2008}
\begin{equation}
    \gHkp^{(0)} := \Set{h = \sum_{i=1}^{s_h} \alpha_i \kappa(\rho_i, \cdot) \;:\;  \alpha_i \in \sR, \; \rho_i \in \gSH, \; s_h \in \sN},
\end{equation}
i.e., $\gHkp := \overline{\gHkp^{(0)}}$. We equip such an RKHS space with the norm defined for each $h \in \gHkp$ as $\normkp{h} := \sqrt{\Braket{h, h}_\kappa}$, with $\Braket{\cdot, \cdot}_\kappa$ being the RKHS's inner-product~\cite{Hofmann_Schölkopf_Smola_2008}.

This readout scheme adheres to the principle of projected quantum kernels~\cite{Huang_Broughton_Mohseni_Babbush_Boixo_Neven_McClean_2021}, consisting of first measuring and then applying a kernel wrapper to the resulting feature vectors. Hence, our approach is a legitimate quantum kernel method applied to time series data. Additionally, we choose a Mat\'ern profile because its RKHS exhibits a polynomial algebraic structure, which is required to preserve the polynomial-algebra structure induced by the spatial multiplexing design used to increase model expressivity~\cite{Monzani_Prati_2025}.

\subsection{Readout training with empirical risk minimization}\label{subsec:readout-training}
Given a fixed quantum channel $T$, i.e.\ after sampling parameters $\vartheta$, and fixing the number of qubits $n$ and the number of sub-reservoirs $R$, we now show how the kernel-based readout $h$ is trained. 
Let $\varphi$ be a Mat\'ern profile, as described in~\cref{eq:matern-profile}, parameterized by hyperparameters $\nu$ and $\xi$. Suppose we have a (training) dataset of windows $\gD = \{(\rvx_1, y_1), \dots, (\rvx_N, y_N) \}$ that contains window samples $\rvx_i$ along with their labels $y_i \in \sR$, obtained from some unknown functional $H^\star$. We learn the best readout $h \in \gHkp$ following an empirical risk minimization scheme, where we minimize the empirical risk $\hat{R}_N(H^T_h)$ with respect to the dataset $\gD$ given by
\begin{equation}\label{eq:empirical-risk-proxy}
    \hat{R}_N(H^T_h) := \frac{1}{N} \sum_{i=1}^N \ell\pth{H^T_h(\rvx_i), y_i},
\end{equation}
with $\ell$ being the squared loss function\footnote{Other Lipschitz losses may be used as discussed in Section~\ref{subsec:learning-temporal-data}.} $\ell(\hat{y}, y) := (\hat{y} - y)^2$.

The bias-variance tradeoff of the model is controlled by restricting the RKHS readout. Concretely, $h$ is learned by minimizing the empirical risk while enforcing a norm budget $\normkp{h} \leq \Lambda$. Such a constrained problem admits an equivalent Tikhonov-regularized formulation~\cite{scholkopf2001learning} where there exists a $\lambda_{\rm \reg} \geq 0$ (a Lagrange multiplier for the constraint $\normkp{h} \le \Lambda$) for which the solution of the norm-constrained empirical risk minimization satisfies
\begin{equation}
    \argmin_{h\in \gHkp(\Lambda)} \hat{R}_N(H^T_h) = \argmin_{h \in \gHkp} \hat{R}_N(H^T_h) + \lambda_{\rm \reg} \normkp{h}^2,
\end{equation}
where $\gHkp(\Lambda)$ is the RKHS ball of radius $\Lambda$ defined as 
\begin{equation}
    \gHkp(\Lambda) := \{h \in \gHkp : \normkp{h} \leq \Lambda\}
\end{equation}
By the representer theorem~\cite{Hofmann_Schölkopf_Smola_2008}, the optimal readout admits the form $h(\cdot) = \sum_{i=1}^N \alpha_i \kappa\pth{H^T(\rvx_i), H^T(\cdot)}.$ Training thus reduces to kernel ridge regression, where we solve a finite-dimensional linear system in $\valpha$. The closed-form solution is given by
\begin{equation}
    \valpha = \pth{K + N \lambda I}^{-1} \vy,
\end{equation}
where $K$ is the Gram matrix with elements $K_{ij} := \kappa\big(H^T(\rvx_i), H^T(\rvx_j)\big)$ for the different windows $\rvx_i, \rvx_j \in \gD$, and $\vy := (y_1, \ldots, y_N)^\top$ is the label vector.

%% file: sections/generalization.tex
\subsection{Learning theoretic analysis}\label{subsec:learning-theoretic-analysis}
\subsubsection{Effective task learning}\label{subsec:effective-learning}
In this section, we show that we can effectively learn the task functional under mild conditions on the quantum resources used. Let our dataset of windows be $\gD = \{(\vx_1, y_1), \dots, (\vx_N, y_N) \}$. We argue as follows that, upon using a quantum reservoir computer with a number of qubits as low as $n = \Omega(\log(w N))$, we can, with high probability, achieve effective empirical learning. In this setting, we show in the following theorem that, in an idealized noiseless setting (as opposed to the noisy feature maps of Section~\ref{subsubsec:efficient-msmt-scheme}), we obtain effective learning of the task at hand. The proof of such result can be found in Appendix~\ref{subapp:proof-thm-effective-learning}.

\begin{theorem}[Effective learning]\label{thm:effective-learning}
    Suppose a design of a quantum reservoir as described in Section~\ref{subsec:QR-embedding} with a number of $R \geq 1$ sub-reservoirs, and let our readout scheme follow the kernel-based one described in Section~\ref{subsec:K-read}. Let $\varepsilon_{\pr}, \delta_{\pr}  \in (0,1)$ be, respectively, an error tolerance and a failure probability parameter. Then, with a number of qubits as low as $n = \Omega\pth{\varepsilon_{\pr}^{-2} \log\pth{\tfrac{w N}{\delta_{\pr}}}}$, with probability at least $1-\delta_{\pr}$, there exists a value $\Lambda^\star > 0$ such that for any constraint $\Lambda \leq \Lambda^\star$, we have
    \begin{equation}\label{eq:empirical-risk-and-RKHS-norm}
        \min_{h\in \gHkp(\Lambda)} \hat{R}_N(H^T_h) = O\pth{\big(1 - \tfrac{\Lambda}{\Lambda^\star}\big)^2}.
    \end{equation}
\end{theorem}

Theorem~\ref{thm:effective-learning} can be read as a capacity statement for the end-to-end functional $H^T_h$: for a fixed quantum featurizer, enlarging the RKHS norm constraint decreases the achievable training loss. This is made explicit in~\cref{eq:empirical-risk-and-RKHS-norm} by exhibiting a budget threshold (which depends on the Gram matrix) such that, once $\Lambda$ is relaxed enough to approach $\Lambda^\star$, the empirical risk can be driven arbitrarily close to zero. Therefore, under the stated qubit scaling (which leads to a JL projection preserving the finite window set's geometry with high probability), the kernel readout becomes expressive enough to interpolate the dataset under sufficiently weak regularization. The interpretation of this result should be solely that there exists an effective interpolation regime.

\subsubsection{Generalization analysis on weakly-dependent data}\label{subsubsec:generalization}
Though Theorem~\ref{thm:effective-learning} shows that empirical fit is achievable with modest quantum resources, it says nothing about how the empirical risk deviates from the true risk $R(H^T_h)$ as defined in~\cref{eq:risk-def}, i.e., generalization is not guaranteed. Characterizing the generalization gap $\abs{R(H^T_h) - \hat{R}_N(H^T_h)}$ is, in general, a non-trivial endeavour for time series data, as we do not necessarily have the required i.i.d.\ property for the window samples of our hookups dataset $\gD$ (between-window dependence may occur), which is often assumed by standard learning-theoretic analyses.

In our treatment, we provide a generalization characterization for the case where our data are weakly dependent and are sampled from a $\beta$-mixing process $\rmIO = ((X_t, Y_t) : t \in \sZ_-)$. In essence, a $\beta$-mixing process means that the ``far past'' of the sequence and the ``far future'' become more and more independent as we increase the temporal gap (formal definitions can be found in Appendix~\ref{subsec:weakly-dependent-data}). This property can be leveraged to enable coupling/blocking arguments that lead to generalization bounds similar to the i.i.d.\ case, with an explicit dependence-related penalty~\cite{mohri2008rademacher}. It is worth mentioning that this weak-dependence assumption encompasses a broad class of time series models used in practice, which are $\beta$-mixing under verifiable stability/ergodicity conditions, making our analysis applicable well beyond the i.i.d.\ setting~\cite{dedecker2007,Doukhan_1994,Carrasco2002}.

We now describe how we construct a dataset that leads to our generalization guarantee. We observe a stationary input/output process $\rmIO$ (see Section~\ref{subsec:learning-temporal-data} for details). Then, a window $\rvx_i$ is constructed as the realization of the sub-process $(X_{t_i - w + 1}, \ldots, X_{t_i})$ and is assigned a label $y_i$, the realization of $Y_{t_i}$. When windows overlap, dependence arises; hence, we instead select indices with stride $s = w + g$, where $g$ is a controllable gap. Increasing $g$ increases, in turn, the independence between windows for our $\beta$-mixing process. Sampling $N$ windows while respecting stride $s$ between consecutive windows is how we construct our dataset $\gD = \{(\rvx_1, y_1), \ldots, (\rvx_N, y_N)\}$, where we assume $t_1 > t_2 > \cdots > t_N$ with $t_i - t_{i+1} = s$. We now present our generalization result for the $\beta$-mixing case the proof of which can be found in Appendix~\ref{subsubsec:proof-thm-generalization}.

\begin{theorem}[Generalization on weakly-dependent data]\label{thm:generalization}
Consider $\rmIO = ((X_t,Y_t):t)$ as a stationary $\beta$-mixing process with bounded outputs
$|Y_t|\leq \Upsilon_Y$. Fix a window length $w$ and a gap $g$, and construct a strided windows
dataset $\gD=\{(\rvx_i,y_i)\}_{i=1}^N$ with stride $s=w+g$ (as in the paragraph above). Assume $N$ is even and write $N=2\mu$. Let $T$ be a QR featurizer composed of $R$ sub-reservoirs with contraction values $\{\lambda_r\}_{r=1}^R$, and let $\kappa$ be the unit-variance Mat\'ern kernel used in Section~\ref{subsec:K-read}
(with parameters $\nu>1$ and $\xi>0$). Let $\gO$ be the set of measured $k$-local observables.
Define $\lambda_\star := \max_{r\in[R]}\lambda_r$.

Fix any $\delta\in(0,1)$ such that
\begin{equation}
    \delta > 4(\mu-1)\beta_{\rmIO}(g),
    \qquad
    \delta' := \delta - 4(\mu-1)\beta_{\rmIO}(g) \;>\; 0.
\end{equation}
Then, with probability at least $1-\delta$, the following holds \emph{simultaneously for all}
$h\in \gHkp(\Lambda)$:
\begin{equation}\label{eq:bound-generalization}
\begin{aligned}
\abs{R(H^T_h) - \hat{R}_N(H^T_h)}
\;\leq\;
&\underbrace{4\sqrt{2}\,\frac{\Lambda(\Lambda+\Upsilon_Y)}{\sqrt{N}}}_{\text{(Rademacher term)}}
\;+\;
\underbrace{3(\Lambda+\Upsilon_Y)^2\frac{\sqrt{\log(4/\delta')}}{\sqrt{N}}}_{\text{(mixing penalty)}}\\
&\;+\;
\underbrace{\frac{4\Lambda(\Lambda+\Upsilon_Y)}{\xi}\sqrt{\frac{\nu R|\gO|}{\nu-1}}\,
\lambda_\star^{\,w}}_{\text{(window truncation via contraction)}}.
\end{aligned}
\end{equation}
\end{theorem}

Theorem~\ref{thm:generalization} provides a uniform high-probability bound on the generalization gap $|R(H^T_h) - \hat{R}_N(H^T_h)|$ for all readouts $h$ within the RKHS ball $\gHkp(\Lambda)$ trained on a strided windows dataset. The bound decomposes into three terms. First, a standard Rademacher complexity term of order $1/\sqrt{N}$, which penalizes model richness. Second, a dependence penalty that is also $O(1/\sqrt{N})$ but whose confidence level is effectively reduced by temporal dependence through $\delta' = \delta - 4(\mu - 1)\beta_{\rmIO}(g)$; hence the need to choose a gap $g$ large enough to reduce $\beta_{\rmIO}(g)$. Third, a fading-memory remainder that decays geometrically with $\lambda_\star^w$, which penalizes the fact that the reservoir forgets information beyond the window horizon. Consequently, for a fixed design, increasing the number of windows $N$ drives the first two terms to zero, while choosing $w$ moderately large (or ensuring stronger contraction) makes the truncation term negligible. In practice, the gap $g$ controls statistical dependence between windows, while $w$ controls the approximation to the reservoir's infinite-memory dynamics, thereby yielding a clear bias--dependence--variance tradeoff.



%% file: sections/numerical-validation.tex
\section{Empirical validation}\label{sec:experiments}
In this section, we provide empirical validation of our end-to-end \textsc{QuaRK} model. The goal is to empirically support our two key learning-theoretic guarantees outlined in our theoretical analysis. First, we study the capacity of our model to effectively learn the selected tasks by showing that the training mean-squared error decreases smoothly and reaches zero beyond a finite threshold as we sweep the readout complexity, thereby corroborating the existence of an interpolation regime for our method. Second, we study our model's generalization capacity. Once trained in this effective-learning regime, we evaluate the out-of-sample performance of the model and demonstrate that the test error decreases as we increase the number of training windows, which aligns with the scaling suggested by our generalization bound. In order to keep the empirical study compact and highly interpretable, we focus on a synthetic $\beta$-mixing vector autoregressive moving average (VARMA) family and representative DGP functionals with varying complexity. The experiments were conducted using Qiskit~\cite{javadi2024quantum}, and the quantum circuits were simulated on a simulator using a machine with two Nvidia A100 GPUs. Code available here \href{https://github.com/abdo-aary/quark}{https://github.com/abdo-aary/quark}.

\subsection{Empirical set up}
For all the experiments, we generate a multivariate input process $\rmX = (X_t)_{t \leq 0} \in \gIZ$, with $\gI := (-1,1)^d$ and $d=3$, from a $\VARMA(p,q)$ recursion with $p = q = 3$ (to allow multistep dependencies), given as
\begin{equation}
    Z_t=\sum_{i=1}^p\Phi_i Z_{t-i}+\Theta_0\epsilon_t+\sum_{j=1}^q \Theta_j \epsilon_{t-j}, \qquad X_t := \tanh(Z_t),
\end{equation}
where $(\epsilon_t)$ are i.i.d.\ centered innovations, and the AR part satisfies the usual stability conditions to allow the $\beta$-mixing property (see Appendix~\ref{subsec:VARMA} for details). From a trajectory, we form supervised examples by sliding windows of length $w$, $\rmX_t := (X_{t-w+1}, \ldots, X_t)$, along with a scalar label $Y_t \in \sR$, with $w = 25$ and a stride $s = 100$ (i.e.\ gap $g = 75$). Labels are generated by fixed ground-truth fading-memory functionals $H^\star: \gIZ \to \sR$, whose window-truncated versions $H^\star_w$ are used so that $Y_t = H^\star_w (\rmX_t)$.

We consider three functionals (tasks) chosen with varying difficulty. The first and easiest task is a one-step forecasting functional $H^\star_{\rm fore}(\boldsymbol{X}_t) := u^\top X_{t+1}$, with $u \in \sR^d$ defining a fixed projection, which depends linearly only on the immediate future. The second, more difficult task is an exponentially fading linear functional, which involves decaying memory over many lags in a linear fashion,
\begin{equation}\label{eq:exponential-fading-linear}
    H^\star_{\rm exp}(\boldsymbol{X}_t)
    := \sum_{k\geq0}\alpha^k\, u^\top X_{t-k},
    \quad
    Y_t := H^\star_{{\rm exp},w}(\boldsymbol{X}_t)
    := \sum_{k=0}^{w-1}\alpha^k\, u^\top X_{t-k},
\end{equation}
with $\alpha \in (0,1)$, thereby requiring the learning model to aggregate information across the whole window while capturing the decaying importance. The final and most difficult task is related to a Volterra-type functional of order two, which adds quadratic cross-lag interactions to the fading memory term in~\cref{eq:exponential-fading-linear},
\begin{equation}\label{eq:volterra-functional}
    H^\star_{\rm vol}(\boldsymbol{X}_t)
    = \sum_{k\geq0}\alpha^k\, u^\top X_{t-k}
    + \frac{1}{2}\sum_{k\geq0}\sum_{\ell\geq0}\alpha^{k+\ell}\,
    (v^\top X_{t-k})(v^\top X_{t-\ell}),
\end{equation}
with $v  \in \sR^d$ defining a fixed projection, and where $H^\star_{{\rm vol}, w}$ is obtained by truncating $k,\ell \in \{0, \ldots, w-1\}$.

\subsection{Learning theory validation}
We used $n = 5$ qubits per sub-reservoir to align the empirical study with the projection-dimension guideline proposed in Section~\ref{subsubsec:projection}. This matches the prescribed scaling $n \approx \log(w N)$ for the JL projection dimension, with $w = 25$ and a maximum number of training windows of $N=8000$. This choice is made for the two experiments below in order to fix the quantum featurizer complexity while probing the emergence of an interpolation regime by sweeping the readout norm budget and testing the generalization bound's sample-size behavior by varying the number of training windows up to $N=8000$.

Embeddings $H^T(\rmX)$ are probed following the classical-shadows measurement strategy on the $2$-local Pauli observable set. We perform $1000$ measurement shots (classical snapshots) per circuit to construct the feature vector $m \circ H^T(\rmX)$. 
A number of $R=3$ independent sub-reservoirs are used to obtain a rich representational capacity. Throughout all the experiments, the architecture of each sub-reservoir is fixed to be the ring architecture and unitary pattern illustrated in Figure~\ref{fig:unitary-evolution}.

\subsubsection{Effective learning validation}
\begin{figure}
    \centering
    \includegraphics[width=0.8\linewidth]{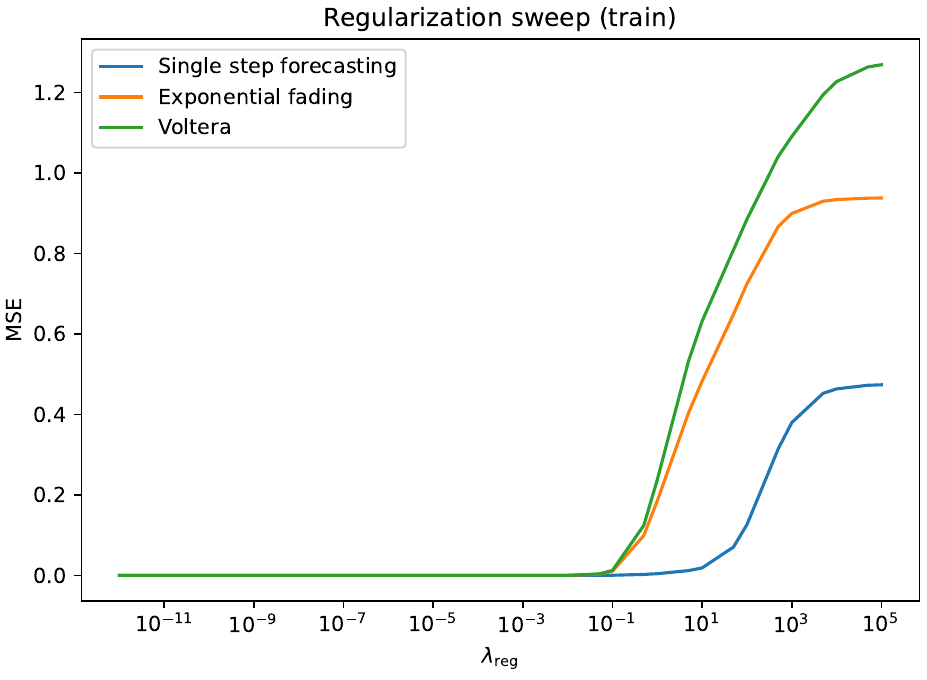}
    \caption{\small Training MSE versus the readout regularization $\lambda_{\rm reg}$ for the three functionals, featuring a sharp transition into the interpolation regime around $\lambda_{\rm reg} \approx 10^{-1}$, where the curves reach numerical zero.}
    \Description{Line plot of training mean-squared error versus the readout regularization strength on a logarithmic horizontal axis. Three curves are shown: single-step forecasting, exponential fading, and Volterra. For very small regularization, all curves stay near zero error (interpolation regime). Around a regularization value near 0.1, the error increases sharply. The Volterra curve rises the most, exponential fading rises next, and single-step forecasting remains the lowest among the three.}
    \label{fig:regularization-sweep}
\end{figure}

The first experiment validates the effective-learning prediction of our theoretical analysis in Section~\ref{subsec:effective-learning} by isolating a single control parameter: the readout complexity. Concretely, we fix all experimental choices, including the quantum featurizer, and we perform a sweep over the regularization parameter $\lambda_{\rm reg}$ (equivalently, a sweep over the RKHS norm budget $\Lambda$). Figure~\ref{fig:regularization-sweep} reports the training MSE for each of the three tasks as a function of the regularization parameter. The key observation in this figure is the existence of a clear transition point into an interpolation regime, where, as we relax the constraint, the training error decreases and eventually reaches numerical zero in a smooth fashion. This is precisely what is predicted by Theorem~\ref{thm:effective-learning}: beyond a finite threshold, the readout hypothesis class becomes rich enough to push the empirical risk to zero. Furthermore, Figure~\ref{fig:train-fit} illustrates the fit on a subset of training windows in the interpolation regime, where we visualize predictions against ground-truth labels. This shows perfect fitting of the training dataset and highlights that the low training loss reported in Figure~\ref{fig:regularization-sweep} is not merely an artifact of averaging.

\begin{figure}
    \centering
    \includegraphics[width=\linewidth, height=6cm]{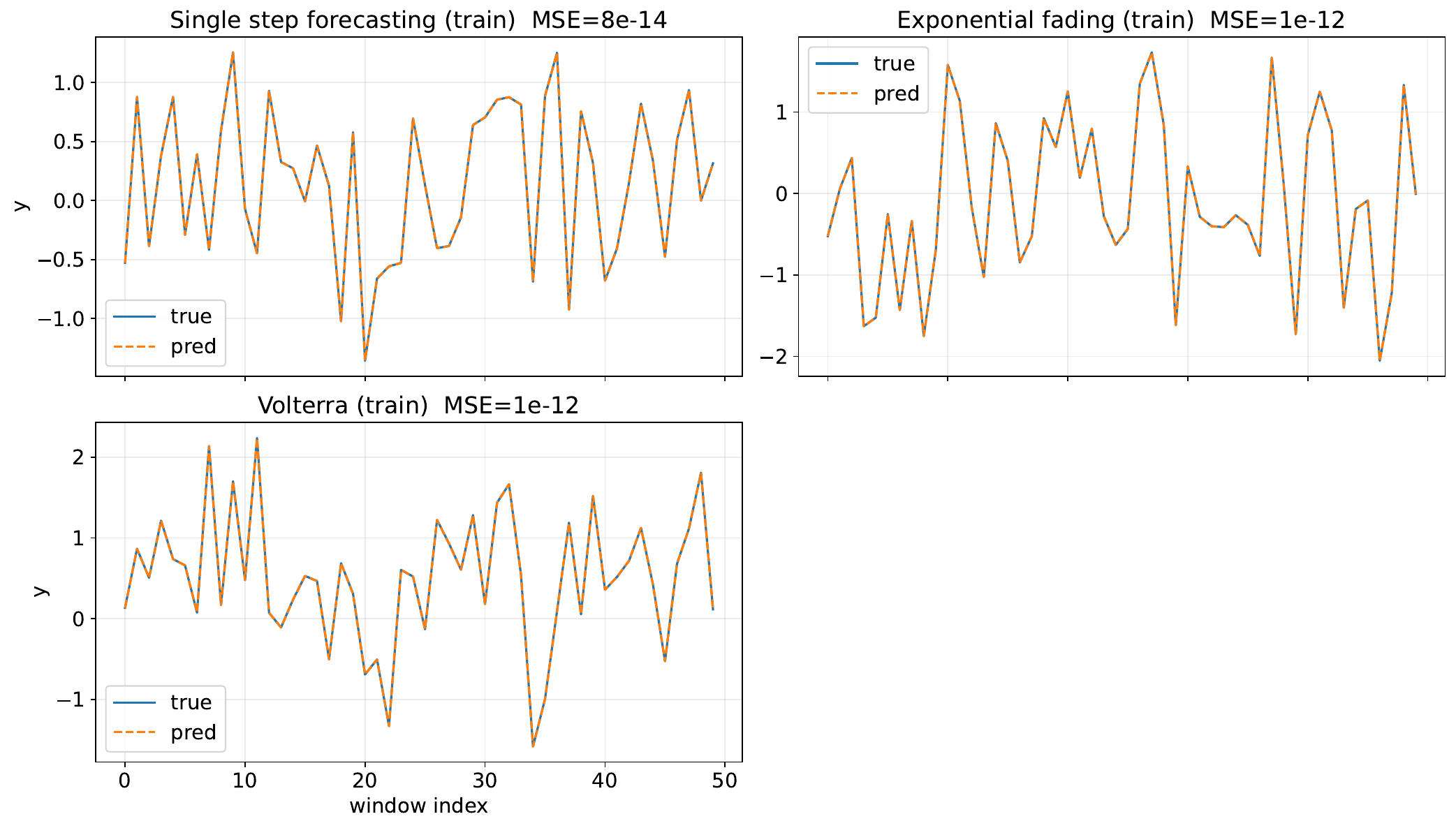}
    \caption{\small Predicted vs true labels on a subset of training windows for the three functionals from the interpolation regime. Numerical errors reach values of ${\rm MSE} \in [10^{-14}, 10^{-12}]$ confirming perfect task learning.}
    \Description{Three time-series panels comparing ground-truth outputs (“true”, solid line) to model predictions (“pred”, dashed line) over window index (about 0 to 50) for single-step forecasting, exponential fading, and Volterra tasks. In all panels, the predicted curve is visually indistinguishable from the true curve, indicating near-perfect fit on the training windows, consistent with the extremely small reported training MSE values in the panel titles.}
    \label{fig:train-fit}
\end{figure}

\subsubsection{Generalization guarantee validation}\label{subsubsec:generalization-validation}
In this second experiment, we experimentally validate the claim made in Theorem~\ref{thm:generalization}. We first fix all modeling choices to a single configuration: same kernel, same quantum featurizer, and a single fixed readout regularization $\lambda_{\rm reg}$ chosen to be sufficiently relaxed. Then, we sweep the number of training windows $N$ from $100$ to $8 \cdot 10^3$, where for each value of $N$, we train the same kernel readout on the corresponding training subset and record the out-of-sample MSE on a fixed held-out test set. Figure~\ref{fig:varying-datasizes} illustrates that the test error decreases as $N$ increases across the three tasks, which aligns with the scaling predicted by the bound in~\cref{eq:bound-generalization}. Importantly, the decay of the three curves follows the same $\tfrac{1}{\sqrt{N}}$ shape (up to additional constants), thereby validating the prediction made in Section~\ref{subsubsec:generalization}. Furthermore, the ranking of the obtained test MSE values reflects the relative hardness of the functionals: the single-step forecasting functional yields the lowest values, followed by the exponential-fading and Volterra functionals.

\begin{figure}
    \centering
    \includegraphics[width=\linewidth, height=6cm]{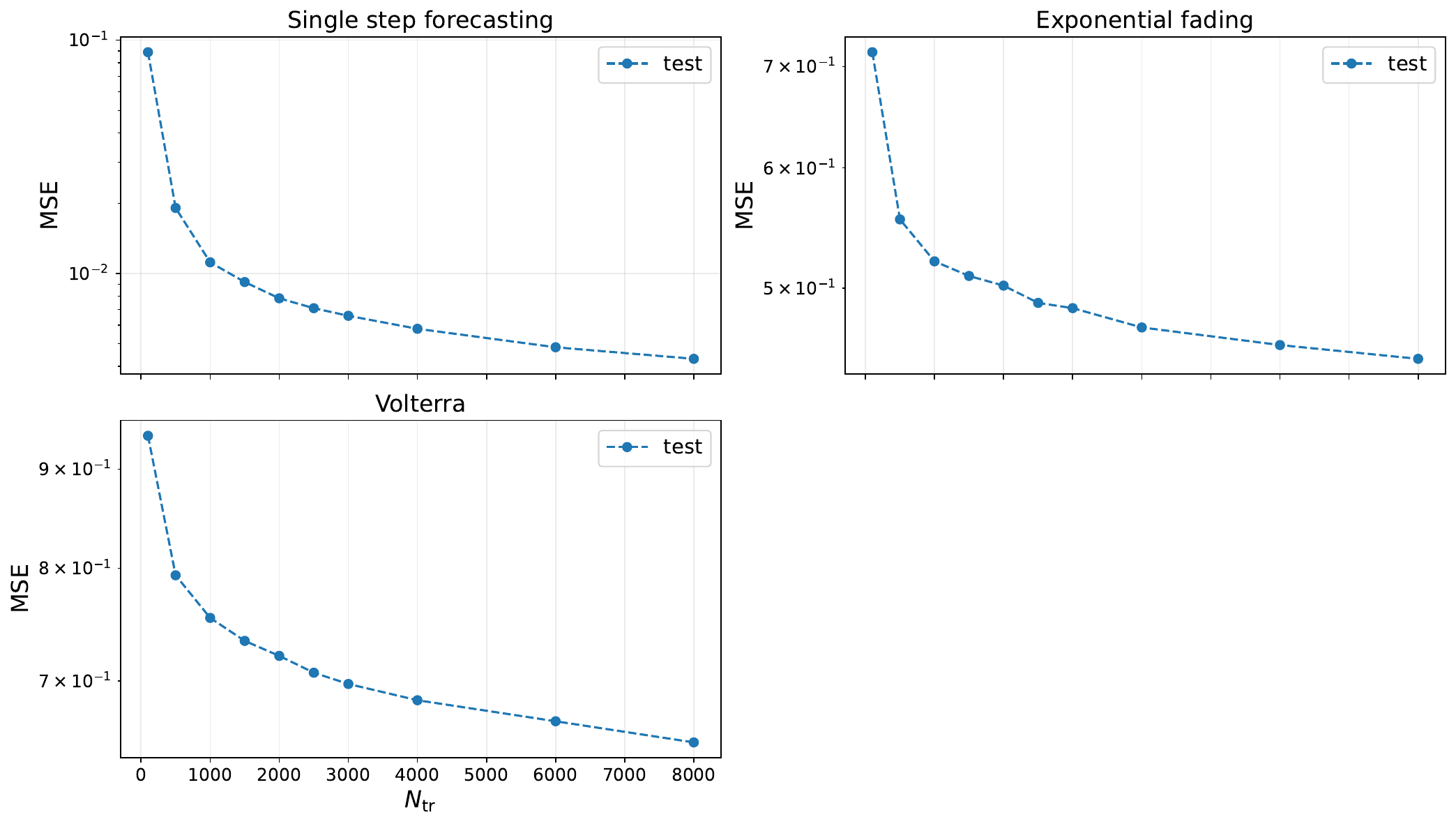}
    \caption{\small Test MSE vs training-set size $N$ reported for the three tasks, obtained by sweeping $N$ while keeping the other experimental choices, and reported for a same held-out test split.}
    \Description{Three panels showing test MSE as a function of training-set size N for single-step forecasting, exponential fading, and Volterra. In each panel, a dashed line with markers decreases as N increases, indicating improved generalization with more training data. The single-step task shows the steepest drop, while exponential fading and Volterra decrease more gradually but consistently over the full range of N.}
    
    \label{fig:varying-datasizes}
\end{figure}

%% file: sections/conclusion.tex
\section{Conclusion} \label{sec:conslusion}
In this paper, we introduced \textsc{QuaRK}, an end-to-end quantum--classical learner for time-series prediction that couples a contractive quantum reservoir featurizer with a kernel-based RKHS readout. Our pipeline operates on windowed inputs and leverages a Johnson--Lindenstrauss projection to decouple the required quantum resources from the ambient data dimension. To enhance expressiveness while keeping the quantum component hardware-realistic, \textsc{QuaRK} relies on spatial multiplexing through multiple sub-reservoirs, and forms compact feature vectors by estimating families of $k$-local Pauli observables using classical-shadow tomography; these features are then learned with kernel ridge regression, where regularization provides an explicit and interpretable control of readout complexity.

On the learning-theoretic side, we connected these design choices to finite-sample performance under weak dependence. We established an effective interpolation regime, showing that once the embedding preserves enough geometry (via the prescribed projection/qubit scaling), relaxing the RKHS constraint yields a class rich enough to drive the empirical risk to (numerically) zero. We then provided a PAC-style generalization result for strided windows extracted from a stationary $\beta$-mixing process, recovering i.i.d.-like rates up to an explicit dependence penalty; empirically, our synthetic $\beta$-mixing VARMA benchmarks confirm the interpolation transition and the expected test-error decay when sweeping the number of training windows with a fixed featurizer and kernel.

Several directions remain to further consolidate \textsc{QuaRK} as a practical and hardware-relevant methodology. A natural next step is to sharpen the analysis under realistic noise and finite-shot effects, and to make the tradeoffs between measurement budget, observable locality $k$, multiplexing, and accuracy more explicit. Finally, more systematic tuning of the design choices (projection dimension/qubit count, sub-reservoir count, observable families, and Mat\'ern hyper-parameters) and validation on higher-dimensional, longer-memory, and nonstationary real-world time series will clarify when each component of the pipeline (projection, contraction, multiplexing, and kernel readout) yields the strongest gains.

%% file: appendices/analysis.tex
\section{Material related to learning-theoretic analysis}
\subsection{Proof of Theorem~\ref{thm:effective-learning}}\label{subapp:proof-thm-effective-learning}
Before we provide the proof of this theorem, we first describe in Section~\ref{subsubsection:useful-results} how a state $U^T(\rvx)_t =: \rho_t$ depends on previous states in Claim~\ref{cl:closed-form-recursion}, followed by a set-injectivity result in Proposition~\ref{prop:set-injectivity} that we leverage to establish the proof of Theorem~\ref{thm:effective-learning}.

\subsubsection{Useful results}\label{subsubsection:useful-results}
\begin{claim}[Closed form reservoir recursion]\label{cl:closed-form-recursion}
Let $\rho_+ = \pdensity$ and define the input-dependent channel $\gJ_{x_t}(\rho) := \gJ(\rho, x_t)$ for elements $x_t$ of a window $\rvx = (x_{\tau-w+1}, \ldots, x_{\tau})$. Then the CEQC recursion of~\cref{eq:CEQC} can be written as
\begin{equation*}
    \rho_t = \lambda \gJ_{x_t}(\rho_{t-1}) + (1-\lambda) \rho_+.
\end{equation*}
Define for $k \geq 1$ the backward composition operator
\begin{equation*}
    \gJarr{k}_{t}(\rho | \rvx) := \pth{\gJ_{x_t} \circ \gJ_{x_{t-1}} \circ \cdots \circ \gJ_{x_{t-k+1}}} (\rho),
\end{equation*}
and for $k = 0$, $\gJarr{k}_{t}(\rho|\rvx) := \rho$. Then
the reservoir's state after consuming inputs $(x_{\tau - w + 1}, \dots, x_t)$ becomes
\begin{equation}\label{eq:recursion}
    \rho_t = U^T(\rvx)_t = \lambda^m \gJarr{m}_{t}(\rho_{\tau-w} | \rvx) + (1-\lambda) \sum_{k=0}^{m-1} \lambda^k \gJarr{k}_{t}(\rho_+ | \rvx),
\end{equation}
with $m$ being the number of applications of the evolution $T$: $m = t - (\tau - w) \in \{1, \ldots, w\}$.
\end{claim}
\begin{proof}
    By induction on $m$. The case $m=1$ is~\cref{eq:CEQC}. Assuming the formula holds at time $t-1$, substitute it into $\rho_t = \lambda \gJ_{x_t}(\rho_{t-1}) + (1-\lambda) \rho_+$ and use the definition of $\gJarr{k}_{t}$ to obtain~\cref{eq:recursion}.
\end{proof}

\begin{proposition}[Set-injectivity]\label{prop:set-injectivity}
    Let $\gD = \{\rvx_1, \ldots, \rvx_N\}$ be a dataset of windows. Suppose our quantum featurizer (initialized at $\rho_{-w} := \pdensity$) contains $R \geq 1$ sub-reservoirs that follow the design described in Section~\ref{subsec:QR-embedding}. Let our readout scheme follow the kernel-based one described in Section~\ref{subsec:K-read}. Let $\varepsilon_{\pr}, \delta_{\pr}  \in (0,1)$ be, respectively, a tolerance and a failure probability choice. Then, with a number of qubits as low as $n = \Omega\pth{\varepsilon_{\pr}^{-2} \log\pth{\tfrac{w N}{\delta_{\pr}}}}$, we have $m \circ H^T(\rvx) \neq m \circ H^T(\rvx')$ for any pair of windows $\rvx \neq \rvx' \in \gD$ with probability at least $1 - \delta_{\pr}$.
\end{proposition}

\begin{proof}
    In this proof, we only treat the case $R = 1$, as the cases $R > 1$ follow directly.
    
    Let $\gP = \bigcup_{i=1}^N \rvx_i \subset \sR^d$ be the set that contains all points that appear inside any window, which has cardinality at most $|\gP| \leq wN$. As described in Section~\ref{subsubsec:projection}, the event $\mathrm{E_{JL}}$ that describes that, for all $u, v \in \gP$,
    \begin{equation*}
        (1-\varepsilon_{\pr}) \normeuc{u - v}^2 \leq \normeuc{\Pi u - \Pi v}^2 \leq (1+\varepsilon_{\pr}) \normeuc{u - v}^2,
    \end{equation*}
    has probability $\Pr(\mathrm{E_{JL}}) \geq 1- \delta_{\pr}$ when $\Pi$ follows a Gaussian Johnson--Lindenstrauss distribution as stated in~\cref{eq:JL-matrix} and $n = \Omega\pth{\varepsilon_{\pr}^{-2} \log\pth{\tfrac{w N}{\delta_{\pr}}}}$, where the probability is over $\Pi$.
    Therefore, on the event $\mathrm{E_{JL}}$, the map $u \mapsto \Pi u$ is injective on the set $\gP$. From now on, condition on $\mathrm{E_{JL}}$.

    Let $\rvx = (x_{\tau - w + 1}, \ldots, x_\tau)$ and $\rvx' = (x'_{\tau' - w + 1}, \ldots, x'_{\tau'})$ be two distinct windows from the dataset $\gD$. To alleviate clutter, we use the re-indexing by relative lag $t \in \{-w+1, \ldots, 0\}$:
    \[
        x_t := x_{\tau + t}, \qquad x'_t := x'_{\tau' + t}.
    \]
    Since $\rvx$ and $\rvx'$ are distinct, let $t_\star$ be the latest lag where they differ:
    \begin{equation*}
        t_\star := \max\{t\in\{-w+1,\ldots,0\}: x_t \neq x'_t\}.
    \end{equation*}
    Then for all $t>t_\star$, we have $x_t=x'_t$. On the event $\mathrm{E_{JL}}$, $\Pi$ is injective on $\gP$, hence $z_{t_\star}\neq z'_{t_\star}$, where $z_t:=\Pi x_t$ and $z'_t:=\Pi x'_t$. In particular, there exists at least one coordinate $j_\star$ such that $(z_{t_\star})_{j_\star} \neq (z'_{t_\star})_{j_\star}$. Additionally, injectivity of the nonlinearity $\theta(u) = \pi \tanh(u)$ preserves this difference, so that $\theta\big((z_{t_\star})_{j_\star}\big) \neq \theta\big((z'_{t_\star})_{j_\star}\big)$.
    
    Let $\vartheta$ collect all random angles in the Ising unitary as well as the parameter $\lambda$ in~\cref{eq:contractive-channel}. Define the ``collision'' condition for this window pair $(\rvx, \rvx')$ as
    \begin{equation*}
        \gB_{\rvx, \rvx'}  := \{\vartheta: \Phi_{\vartheta}(\rvx) = \Phi_{\vartheta}(\rvx')\}, \quad \Phi_{\vartheta}(\rvx) := m \circ H^{T_\vartheta}(\rvx).
    \end{equation*}
    All that is left in the proof is to show that for each fixed pair $\rvx \neq \rvx'$, the set $\gB_{\rvx, \rvx'}$ has Lebesgue measure zero. A finite union argument then yields that the union of $\gB_{\rvx, \rvx'}$ over all $\binom{N}{2}$ pairs is still a measure-zero set; hence, with probability $1$ over $\vartheta$ (since $\vartheta$ is absolutely continuous with respect to. the Lebesgue measure under i.i.d.\ uniform sampling), no collisions occur on $\gD$ on the event $\mathrm{E_{JL}}$.

    Consider the scalar function
    \begin{equation*}
        g_{\rvx, \rvx'}(\vartheta) := \normeuc{\Phi_{\vartheta}(\rvx) - \Phi_{\vartheta}(\rvx')}^2.
    \end{equation*}
    Each feature coordinate $\Trc{O\,H^{T_\vartheta}(\rvx)}$ is obtained by composing finitely many maps that depend real-analytically on $\vartheta$ (products of single-/two-qubit rotations) with affine linear CPTP maps, and then taking a trace against a fixed observable. Hence $\Phi_{\vartheta}(\rvx)$ and therefore $g_{\rvx,\rvx'}(\vartheta)=\|\Phi_{\vartheta}(\rvx)-\Phi_{\vartheta}(\rvx')\|_2^2$ are real-analytic in $\vartheta$. Therefore, if we show a parameter choice $\vartheta^\star$ such that $g_{\rvx, \rvx'}(\vartheta^\star) > 0$, it follows from real-analyticity of $g_{\rvx, \rvx'}$ that the set $\{\vartheta: g_{\rvx, \rvx'}(\vartheta) = 0\}$ has Lebesgue measure zero (as can be found in Proposition~0 of~\cite{mityagin2015zeroset}). We argue in what follows the existence of such a witness.
    
    Let $\vartheta^\star_\bullet = 0$. Then we have 
    \begin{equation*}
        V(x) = \pth{\bigotimes_{j=1}^n R_y(\theta(z_j))}.
    \end{equation*}
    As each injection qubit undergoes only $Y$-rotations, if the two windows differ at the latest time $t_\star$ and projected coordinate $j_\star$, then the injected angles differ by
    \begin{equation*}
        \Delta \theta = \underbrace{\theta\big((z_{t_\star})_{j_\star} \big)}_{\alpha_{t_\star}} - \underbrace{\theta\big((z'_{t_\star})_{j_\star} \big)}_{\alpha'_{t_\star}} \neq 0.
    \end{equation*}
    Since $t_\star$ is the latest lag at which the windows differ, we have $\alpha_t = \alpha'_t$ for all $t > t_\star$, hence the accumulated angle $\phi := \sum_{t_\star + 1}^0 \alpha_t = \sum_{t_\star + 1}^0 \alpha'_t$. Consider the $1$-local observables $X_{j_\star}$ and $Z_{j_\star}$ at the level of qubit $j_\star$ (these belong to the feature vector since the measured set of $k$-local observables $\gO$ contains all $1$-local Paulis as discussed in Section~\ref{subsubsec:efficient-msmt-scheme}). Denote the sub-vectors
    \begin{equation*}
        \mu(\rvx) := \pth{\Braket{X_{j_\star}}_{\rvx}, \Braket{Z_{j_\star}}_{\rvx}} \in \sR^2, \quad \mu(\rvx') := \pth{\Braket{X_{j_\star}}_{\rvx'}, \Braket{Z_{j_\star}}_{\rvx'}} \in \sR^2.
    \end{equation*}
    Since
    $$
    \normeuc{\mu(\rvx) - \mu(\rvx')} \neq 0 \implies \|\Phi_{\vartheta}(\rvx)-\Phi_{\vartheta}(\rvx')\|_2 \neq 0,
    $$
    showing a witness $\lambda_\star \in (0,1)$ for which $\mu(\rvx) - \mu(\rvx') \neq 0$ establishes the proof of the proposition. This is precisely what we show in what follows.

    Leveraging the CEQC recursion defined in Claim~\ref{cl:closed-form-recursion}, the output state $H^T(\rvx)$ at lag $0$ admits the expansion
    \begin{equation*}
        H^T(\rvx) = \lambda^w \gJarr{w}_{0}(\rho_{-w}|\rvx) + (1-\lambda) \sum_{k=0}^{w-1} \lambda^k \gJarr{k}_{0}(\rho_+|\rvx),
    \end{equation*}
    where $\gJarr{k}_{0}(\cdot|\rvx)$ applies the last $k$ encoding steps (from lags $-k+1$ up to $0$) starting from state $\rho$. Define $i_\star := -t_\star + 1$. Because $\rvx$ and $\rvx'$ coincide for all lags $t > t_\star$, it follows that for all $k \leq i_\star - 1$,
    \begin{equation*}
        \gJarr{k}_{0}(\rho_+|\rvx) = \gJarr{k}_{0}(\rho_+|\rvx').
    \end{equation*}
    Hence the final state difference $\Delta \rho_0 := H^T(\rvx) - H^T(\rvx')$ can be written as
    \begin{equation}\label{eq:diff-state}
    \begin{aligned}
        \Delta \rho_0 = \; &\lambda^w \pth{\gJarr{w}_{0}(\rho_{-w}|\rvx) - \gJarr{w}_{0}(\rho_{-w}|\rvx')}\\
        &+ (1-\lambda) \lambda^{i_\star} \pth{\gJarr{i_\star}_{0}(\rho_+|\rvx) - \gJarr{i_\star}_{0}(\rho_+|\rvx')}\\
        &+ (1-\lambda) \sum_{k=i_\star + 1}^{w-1} \lambda^k \pth{\gJarr{k}_{0}(\rho_+|\rvx) - \gJarr{k}_{0}(\rho_+|\rvx')}.
    \end{aligned}
    \end{equation}
    Applying the linear map $\sigma \mapsto (\Trc{X_{j_\star}\sigma},\,\Trc{Z_{j_\star}\sigma})$ to~\cref{eq:diff-state}
    and using linearity of the trace yields
    \begin{equation}\label{eq:delta-mu-decomp}
    \mu(\rvx)-\mu(\rvx')
    =
    \lambda^w \Delta\mu^{(w)}_{\rho_{-w}}
    +
    (1-\lambda)\lambda^{i_\star} \Delta\mu^{(i_\star)}_{+}
    +
    (1-\lambda)\sum_{k=i_\star+1}^{w-1}\lambda^{k} \Delta\mu^{(k)}_{+},
    \end{equation}
    where, for any $k\ge 0$ and any starting state $\rho$,
    \[
    \begin{aligned}
        \Delta\mu^{(k)}_{\rho}
    :=
    \Big(
    &\Trc{X_{j_\star} \gJarr{k}_{0}(\rho|\rvx)} -
    \Trc{X_{j_\star} \gJarr{k}_{0}(\rho|\rvx')}, \\
    &\Trc{Z_{j_\star} \gJarr{k}_{0}(\rho|\rvx)}
    -
    \Trc{Z_{j_\star} \gJarr{k}_{0}(\rho|\rvx')}
    \Big).
    \end{aligned}
    \]
    
    As for the witness $\vartheta_\bullet^\star=0$, we have $W=\mathbb I$ and $\gJ_{x_t}(\rho)=V(x_t)\rho V(x_t)^\dagger$ with $V(x_t)=\bigotimes_{j=1}^n R_y(\alpha_t^{(j)})$. On qubit $j_\star$, starting from $\rho_+$ we have
    \[
    \Big(\Trc{X_{j_\star} \gJarr{k}_{0}(\rho_+|\rvx)},
    \Trc{Z_{j_\star} \gJarr{k}_{0}(\rho_+|\rvx)}\Big)
    =
    (\cos(\phi+\alpha_{t_\star}), \sin(\phi+\alpha_{t_\star})),
    \]
    and similarly with $\alpha'_{t_\star}$ for $\rvx'$.
    Hence we have
    \begin{equation}\label{eq:lead-delta}
    \begin{aligned}
        \big\|\Delta\mu^{(i_\star)}_{+}\big\|_2
        \!&=\!
        \normeuc{(\cos(\phi\!+\!\alpha_{t_\star}),\sin(\phi\!+\!\alpha_{t_\star}))
        \!-\!(\cos(\phi+\alpha'_{t_\star}),\sin(\phi+\alpha'_{t_\star}))}\\
        &=
        2\big|\sin(\Delta\theta/2)\big|
        >0.
    \end{aligned}
    \end{equation}

    For any state $\sigma$, $|\Trc{X_{j_\star}\sigma}|\le 1$ and $|\Trc{Z_{j_\star}\sigma}|\le 1$;
    therefore $\|\Delta\mu^{(k)}_{\rho}\|_2\le 2\sqrt{2}$ for all $k$ and all $\rho$.
    Using~\cref{eq:delta-mu-decomp} and the triangle inequality gives, for $i_\star \le w-1$,
    \[
    \begin{aligned}
        \|\mu(\rvx)-\mu(\rvx')\|_2
        \ge\;&(1-\lambda)\lambda^{i_\star}\cdot 2|\sin(\Delta\theta/2)|\\
        &-
        2\sqrt{2}\lambda^w
        -
        2\sqrt{2}(1-\lambda)\sum_{k=i_\star+1}^{w-1}\lambda^{k}.
    \end{aligned}
    \]
    Since $(1-\lambda)\sum_{k=i_\star+1}^{w-1}\lambda^{k}\le \lambda^{i_\star+1} - \lambda^w$, we obtain
    \[
    \begin{aligned}
    \normeuc{\mu(\rvx)-\mu(\rvx')}
    \ge 2(1-\lambda)\lambda^{i_\star}|\sin(\Delta\theta/2)| - 2\sqrt{2}\lambda^{i_\star+1}.
    \end{aligned}
    \]
    Choosing $\lambda_\star \in (0,1/2]$ small enough so that
    \[
    \lambda_\star < \frac{|\sin(\Delta\theta/2)|}{\sqrt{2} + |\sin(\Delta\theta/2)|},
    \]
    yields $\mu(\rvx)\neq\mu(\rvx')$. (If $i_\star=w$, i.e.\ $t_\star=-w+1$, the same computation applies with the leading contribution coming from the $\lambda^w \Delta\mu^{(w)}_{\rho_{-w}}$ term.) This concludes the proof.
    
\end{proof}

\subsubsection{Proof of Theorem~\ref{thm:effective-learning}}\label{subsubsection:proof-thm-effective-learning}
Recall the setup of the K-read training in Section~\ref{subsec:K-read} where, for a fixed reservoir channel $T$, we learn a readout $h \in \gHkp$ by minimizing the empirical risk $\hat{R}_N(H^T_h)$ under an RKHS norm constraint $\normkp{h} \leq \Lambda$, yielding an optimizer admitting a representer depending on the Gram matrix $K$ with elements $K_{ij} = \kappa\pth{H^T(\rvx_i), H^T(\rvx_j)}$ for windows $\rvx_i, \rvx_j \in \gD$. 

Since the conditions of the theorem satisfy those of Proposition~\ref{prop:set-injectivity}, the non-collision event $\mathrm{E_{NC}}$, which represents ``for any windows $\rvx \neq \rvx' \in \gD$, the feature vectors $m \circ H^T(\rvx)$ and $m \circ H^T(\rvx')$ do not collide'', happens with probability at least $1- \delta_{\pr}$. Condition on this event from now on.

As on $\mathrm{E_{NC}}$, the vectors $m \circ H^T(\rvx_i)$ are pairwise distinct and the Mat\'ern kernel is strictly positive definite on distinct inputs, the Gram matrix $K$ is positive definite, hence invertible. Let the label vector $\vy := (y_1, ..., y_N)^\top$, set $\valpha^\star := K^{-1} \vy$, and consider the representer
\begin{equation*}
    h^\star(\cdot) := \sum_{i=1}^N \alpha^\star_i \kappa\pth{H^T(\rvx_i), H^T(\cdot)} \in \gHkp.
\end{equation*}
Then we have $h^\star(H^T(\rvx_i)) = (K \valpha^\star)_i = y_i$ for all $i$, hence $\hat{R}_N(H^T_{h^\star}) = 0$. Additionally, we use $\normkp{h^\star}^2 = (\valpha^\star)^\top K \valpha^\star = \vy^\top K^{-1} \vy$ to define the threshold $\Lambda^\star := \normkp{h^\star}= \sqrt{\vy^\top K^{-1} \vy} > 0$.

Fix any $\Lambda \leq \Lambda^\star$ and let $h_\Lambda$ be the scaled function $h_\Lambda := \tfrac{\Lambda}{\Lambda^\star} h^\star$, which yields $\normkp{h_\Lambda} = \Lambda$. Using the squared loss $\ell(y,y') = (y-y')^2$ in~\cref{eq:empirical-risk-proxy} we get
\begin{equation*}
    \hat{R}_N\pth{H^T_{h_\Lambda}} = \frac{1}{N} \sum_{i=1}^N \pth{h_\Lambda(H^T(\rvx_i)) - y_i}^2 = \pth{1 - \frac{\Lambda}{\Lambda^\star}}^2 \frac{1}{N}\normeuc{\vy}^2.
\end{equation*}
Since $h_\Lambda$ is feasible for the norm-constrained problem, we get that
\begin{equation*}
    \min_{h\in \gHkp(\Lambda)} \hat{R}_N(H^T_h) \leq \pth{1 - \frac{\Lambda}{\Lambda^\star}}^2 \frac{1}{N}\normeuc{\vy}^2 = O\pth{\pth{1 - \frac{\Lambda}{\Lambda^\star}}^2},
\end{equation*}
since, as assumed in Section~\ref{subsec:learning-temporal-data}, the labels are bounded.

Hence, with probability at least $1-\delta_{\pr}$ (that of realization of the event $\mathrm{E_{NC}}$), the above inequality holds, thereby completing the proof.

\qed

\subsection{Weakly-dependent data}\label{subsec:weakly-dependent-data}
In this section, we provide a quick background on $\beta$-mixing processes and outline some results that we use in our proof of Theorem~\ref{thm:generalization}. 

\subsubsection{\texorpdfstring{$\beta$-}{β-}mixing processes}
Let $(\Omega, \gA, \sP)$ be a probability space. For two sub-$\sigma$-algebras $\gU, \gV \subset \gA$, the $\beta$-coefficient (absolute regularity coefficient)~\cite{Volkonskii_Rozanov_Yu_1959} is defined as 
\begin{equation}\label{eq:def-beta-mixing}
    \beta(\gU, \gV) := \frac{1}{2} \sup\acc{\sum_{i=1}^I \sum_{j=1}^J \abs{\sP(U_i \cap V_j) - \sP(U_i)\sP(V_j)}},
\end{equation}
where the supremum is taken over all finite measurable partitions $(U_i)_i$ of $\Omega$ from $\gU$ and $(V_j)_j$ from $\gV$. Equivalently, $\beta$ can be expressed in terms of a norm in total variation~\citep{dedecker2007} 
$$
\beta(\gU, \gV) = \norm{\sP_{\gU \otimes \gV} - \sP_{\gU} \otimes \sP_{\gV}}_{TV},
$$ 
where $\sP_{\gU}, \sP_{\gV}$ denote restrictions of $\sP$ to $\sigma$-fields $\gU, \gV$ and $\sP_{\gU \otimes \gV}$ is a law on the product $\sigma$-field defined on rectangles by $$\sP_{\gU \otimes \gV}(U,V) = \sP(U \cap V).$$

For a stationary random process $\rmX = (X_t)_{t \in \sZ_-}$, the mixing coefficients are obtained as
\begin{equation*}
    \beta_{\scriptscriptstyle \rmX}(k) = \sup_{t^\ast \in \sZ_-} \beta\Big(\sigma(X_t : t \leq t^\ast - k), \sigma(X_t : t \geq t^\ast) \Big),
\end{equation*}
where $\sigma(X_t : t)$ denotes the $\sigma$-field generated by the random process $(X_t : t)$. We say that $\rmX$ is $\beta$-mixing if $\beta_{\rmX}(k) \xrightarrow{k \to \infty} 0$, which means that past--future dependence (regularly) converges to $0$ as we increase the gap $k \in \sN$. Furthermore, $\rmX$ is called geometrically $\beta$-mixing if $\beta_{\rmX}(k) \leq \beta_0 e^{-\beta_1 k}$, and it is called algebraically $\beta$-mixing if $\beta_{\rmX}(k) \leq \beta_0 k^{-\beta_1}$ for some $\beta_0, \beta_1 > 0$.

\subsubsection{\texorpdfstring{$\beta$}{β}-mixing of the window process}\label{subsubsec:beta-mixing-window-process}
Let $w \in \sN$ be a window size, $s = w + g$ for some gap $g \in \sN$, and $\rmIO = ((X_t, Y_t): t \in \sZ_-)$ be a stationary $\beta$-mixing process. Define the windows process $\rmW := \pth{(\rmX_\tau^w, Y_\tau): \tau \in s \sZ_-}$, with $s\sZ_- := \{..., -2s, -s, 0\}$, which consists of dividing the process $\rmIO$ into $w$-sized windows
\begin{equation*}
    \rmX_\tau^w := \pth{X_{\tau - w + 1}, \ldots, X_\tau} \in \gW := (\gI)^w,
\end{equation*}
with indexing jumping by $s$.

The following result shows that the mixing coefficient $\beta_{\rmW}(k)$ of the windows process $\rmW$ is no larger than the mixing coefficient $\beta_{\rmIO}(ks-w)$ of the I/O process $\rmIO$, which we use later on.
\begin{claim}[$\beta$-mixing property of the windows process]\label{cl:beta-mixing-windows-process}
    Let $\rmIO = ((X_t, Y_t) : t \in \sZ_-)$ be the I/O process and $\rmW = \pth{(\rmX^w_\tau, Y_\tau) : \tau \in s \sZ_-}$, with $s = w + g$, for $g \in \sN^\ast$. Then we have
    \begin{equation}
        \beta_{\rmW}(k) \leq \beta_{\rmIO}(k s - w).
    \end{equation}
\end{claim}

\begin{proof}
Let $\rmW := (W_\tau: \tau \in s \sZ_-)$ where $W_\tau := (\rmX_\tau^w, Y_\tau)$. Let $\tau^\star \in s \sZ_-$ and define the past/future $\sigma-$fields for the windows process respectively as
\begin{equation*}
    \gF^- := \sigma\pth{W_\tau: \tau \leq \tau^\star - ks}, \qquad \gF^+ := \sigma\pth{W_\tau: \tau \geq \tau^\star}.
\end{equation*}
As each random window $W_\tau$ is measurable with respect to. $$\sigma\pth{IO_t : t \in \{\tau - w + 1, \ldots, \tau\}},$$ with elements $IO_t := (X_t, Y_t)$, we get 
\begin{equation*}
\begin{aligned}
    \gF^- \subset \sigma\pth{IO_t : t \leq \tau^\star - ks}, \qquad \gF^+ &\subset \sigma\pth{IO_t : t \geq \tau^\star - w + 1}\\
    &\subset \sigma\pth{IO_t : t \geq \tau^\star - w}.
\end{aligned}
\end{equation*}
We use monotonicity of $\beta(\cdot, \cdot)$ of~\cref{eq:def-beta-mixing} under increasing $\sigma-$fields (i.e., if $\gU \subset \gU'$ and $\gV \subset \gV'$, then $\beta(\gU, \gV) \leq \beta(\gU', \gV')$, to deduce that
\begin{equation*}
    \beta(\gF^-, \gF^+) \leq \beta\Big(\sigma\pth{IO_t : t \leq \ \tau^\star - ks}, \sigma\pth{IO_t : t \geq \tau^\star - w}\Big)
\end{equation*}
Setting $t^\star := \tau^\star - w$ we get $\tau^\star - ks = t^\star - (ks - w)$. Hence, we get
\begin{equation*}
\begin{aligned}
    \beta(\gF^-, \gF^+) &\leq \beta\Big(\sigma\pth{IO_t : t \leq \ t^\star - (ks - w)}, \sigma\pth{IO_t : t \geq t^\star}) \\
    &\leq \beta_{\rmIO}(ks - w),
\end{aligned}
\end{equation*}
where the last inequality holds by definition of the $\beta-$mixing coefficients of $\rmIO$. A supremum taken over $t^\star$ yields $\beta_{\rmW}(k) \leq \beta_{\rmIO}(ks - w)$ as claimed.
\end{proof}

\subsubsection{Non-i.i.d. Rademacher generalization bound}\label{subsubsec:non-iid-rademacher-generalization}
As our windows dataset is built from a single realization of the underlying I/O process, the samples $\{(\rvx_i, y_i)\}_{i=1}^N$ (equivalently $\{W_\tau\}$) are in general dependent. A tractable and standard way to account for such dependence is through the $\beta-$mixing coefficients of the window process $\rmW = (W_\tau : \tau \in s \sZ_-)$, as this mixing coefficient captures how quickly far-apart windows appear nearly independent. Non-i.i.d. Rademacher complexity bounds are specifically used for this context to characterize generalization under such setting. 
We provide below an adaptation of Theorem 2 of~\cite{mohri2008rademacher} to our window process $\rmW$.

\begin{theorem}[Mohri–Rostamizadeh~\cite{mohri2008rademacher}]\label{thm:mohri-rostamizadeh}
    Assume that $\rmW = (W_\tau : \tau \in s Z_-)$ is stationary and $\beta-$mixing with coefficients $(\beta_{\rmW}(k))_{k \geq 1}$. Let $\gC$ be a class of measurable functions $f: \gW \to [0, \Upsilon]$. Let $N = 2\mu$ be an even number with $\mu \in \sN^\star$. Define the total windows sample set $S_N$ and odd-indexed windows set $S_\mu \subset S_N$ resp. as
    \begin{equation*}
        S_N := \{\rvw_{-Ns}, \ldots \rvw_{-2s}, \rvw_{-s}\}, \quad S_{\mu} := \{\rvw_{-(2\mu - 1) s}, \ldots \rvw_{-3s}, \rvw_{-s}\}.
    \end{equation*}
    Let the statistical risk and empirical risk  of a hypothesis $f \in \gC$ be defined resp. as
    \begin{equation*}
        R(f) := \E[f(W_0)], \qquad \hat{R}_N(f) := \frac{1}{N} \sum_{i=1}^N f(\rvw_{-is}).
    \end{equation*}
    Let $\delta \in (0,1)$ be such that $\delta > 4 (\mu - 1) \beta_{\rmW}(1)$ and set $$\delta' := \delta - 4 (\mu - 1) \beta_{\rmW}(1).$$
    Then, with probability at least $1-\delta$ (over the draw of $S_N$), we have the following holding simultaneously for all $f \in \gC$:
    \begin{equation}\label{eq:Mohri-Rostamizadeh-risk-bound}
        R(f) \leq \hat{R}_N(f) + \Reb_{N/2}(\gC) + 3 \Upsilon \sqrt{\frac{\log(4/\delta')}{N}},
    \end{equation}
    where $\Reb_{N/2}(\gC)$ is the empirical Rademacher complexity on the subset $S_\mu$ defined as
    \begin{equation*}
        \Reb_{N/2}(\gC) := \E_{\epsilon}\cro{\sup_{f \in \gC} \; \frac{2}{\mu} \; \bigg| \sum_{\rvw \in S_\mu} \epsilon_i \, f(\rvw) \bigg|}, \qquad \epsilon_i \stackrel{\text{i.i.d.}}{\sim}\mathrm{Unif}\{\pm 1\}.
    \end{equation*}
\end{theorem}

\subsection{Proof of Theorem~\ref{thm:generalization}}
The proof of Theorem~\ref{thm:generalization} proceeds as follows.
\begin{enumerate}
    \item Bound the deviation between the window risk and the empirical risk, $\abs{R^w(H^T_h) - \hat{R}_N(H^T_h)}$, where $R^w(H^T_h)$ is the window risk given by
    \begin{equation}\label{eq:window-risk}
        R^w(H^T_h) := \E\cro{\ell\pth{H^T_h(\rmX^w), Y}},
    \end{equation}
    and where $(\rmX^w, Y)$ is a random pair consisting of an input window of size $w$ and its label.
    \item Bound the deviation between the true risk and the window risk, $\abs{R(H^T_h) - R^w(H^T_h)}$. 
    \item Finally, combine the two bounds via the triangle inequality to obtain
    \[
    \abs{R(H^T_h) - \hat{R}_N(H^T_h)}
    \leq
    \abs{R(H^T_h) - R^w(H^T_h)}
    +
    \abs{R^w(H^T_h) - \hat{R}_N(H^T_h)}.
    \]
\end{enumerate}

\subsubsection{Bounding the deviation of the window risk from the empirical one}
Below we show how this deviation behaves as a direct corollary of Theorem~\ref{thm:mohri-rostamizadeh}.

\begin{corollary}[Non-i.i.d.\ generalization for MSE and RKHS-ball readouts]
\label{cor:mse_rkhs_beta}
Let ${\rmIO}=((X_t,Y_t):t\in\sZ_-)$ be stationary and $\beta$-mixing, and let $s=w+g$ with $w,g\in \sN^\ast$.
Construct a strided windows dataset
$\gD=\{(\rvx_i,y_i)\}_{i=1}^{N}$ with $N=2\mu$ even by choosing indices
$t_1>t_2>\cdots>t_N$ such that $t_i-t_{i+1} = s$.
Fix a reservoir embedding $H^T$ and a kernel $\kappa$ as in~\cref{eq:matern-kernel},
and consider the RKHS ball
$\gHkp(\Lambda) = \{h\in\gHkp:\|h\|_\kappa\le \Lambda\}$.
Let $\ell$ be the squared loss $\ell(\hat y,y):=(\hat y-y)^2$.
Assume the labels are bounded $|Y_t|\le \Upsilon_Y$ and that the kernel is normalized so that
$\kappa(\rho,\rho)\le 1$ for all $\rho$ (respected by the Mat\'ern profile, $\kappa(\rho,\rho)=1$).
Let the window risk $R^w(H_h^T)$ be as defined in~\cref{eq:window-risk}.
If $\delta\in(0,1)$ satisfies
\[
\delta>4(\mu-1)\,\beta_{\rmIO}(g),
\qquad
\delta':=\delta-4(\mu-1)\,\beta_{\rmIO}(g),
\]
then with probability at least $1-\delta$, the following holds simultaneously for all
$h\in\gHkp(\Lambda)$:
\begin{equation}
    \abs{R^w(H_h^T)-\widehat R_N(H_h^T)}
    \;\le\; \Reb_{N/2}(\ell \circ \gHkp(\Lambda)) + M(N,w,g).
\end{equation}
where $\Reb_{N/2}(\ell \circ \gHkp(\Lambda))$ denotes the empirical Rademacher complexity of the loss-composed class $\ell \circ \gHkp(\Lambda)$ given by
\begin{equation}
    \Reb_{N/2}(\ell \circ \gHkp(\Lambda)) := \frac{4\Lambda(\Lambda+\Upsilon_Y)}{\sqrt{\mu}} = 4\sqrt{2}\,\frac{\Lambda(\Lambda+\Upsilon_Y)}{\sqrt{N}},
\end{equation}
and the mixing penalty $M(N,w,g)$ is given by
\begin{equation}
    M(N,w,g) := 3(\Lambda+\Upsilon_Y)^2\sqrt{\frac{\log(4/\delta')}{N}}.
\end{equation}
\end{corollary}

\begin{proof}

Let the hypothesis class $\gC := \{f_h : h \in \gHkp(\Lambda)\}$. By the reproducing property and the diagonal bound $\kappa(\rho, \rho) \leq 1$, we have for all $(\rvx,y) \in \gD$
\begin{equation*}
    \abs{h(H^T(\rvx))} \leq \normkp{h}\,\normkp{\kappa(H^T(\rvx), \cdot)} \leq \Lambda.
\end{equation*}
Since $|y| \leq \Upsilon_Y$, it follows that $0 \leq f_h(\rvx, y) \leq (\Lambda + \Upsilon_Y)^2$ for all $(\rvx, y)$, hence $\gC \subseteq [0,\Upsilon]^{\gW \times \sR}$ with $\Upsilon := (\Lambda + \Upsilon_Y)^2$.

Since $t_i-t_{i+1}=s$, the sequence $(\rvw_i)_{i=1}^N:=((\rvx_i,y_i))_{i=1}^N$ is a length-$N$ consecutive segment of the stationary window process $\rmW=(W_\tau:\tau\in s\sZ_-)$ (up to a deterministic time shift).
Therefore, Theorem~\ref{thm:mohri-rostamizadeh} applies directly to the class $\gC$ with mixing coefficient $\beta_{\rmW}(1)$.
Using Claim~\ref{cl:beta-mixing-windows-process}, we upper bound $\beta_{\rmW}(1)\le \beta_{\rmIO}(g)$,
hence $\delta'=\delta-4(\mu-1)\beta_{\rmIO}(g)$ and~\cref{eq:Mohri-Rostamizadeh-risk-bound} yields the bound
\begin{equation*}
    R^w(H_h^T) \leq \widehat R_N(H_h^T) + \Reb_{N/2}(\gC) + 3 \Upsilon \sqrt{\frac{\log(4/\delta')}{N}}.
\end{equation*}
It remains to upper bound the empirical Rademacher complexity $\Reb_{N/2}(\gC) =: \Reb_{N/2}(\ell \circ \gHkp(\Lambda))$. 

Let the odd-indexed sub-dataset be
$\gD^{\rm odd}:=\{(\rvx_{2i-1},y_{2i-1})\}_{i=1}^{\mu}$. By definition (Theorem~\ref{thm:mohri-rostamizadeh}), we have
\begin{equation*}
\begin{aligned}
    \Reb_{N/2}(\gC) = &\E_{\epsilon}\cro{\sup_{h \in \gHkp(\Lambda)} \; \frac{2}{\mu} \; \bigg| \sum_{i=1}^\mu \epsilon_i \, (h(\rho_{2i-1}) - y_{2i-1})^2 \bigg|}, \\
    &\epsilon_i \stackrel{\text{i.i.d.}}{\sim}\mathrm{Unif}\{\pm 1\},
\end{aligned}
\end{equation*}
where $\rho_{2i-1} := H^T(\rvx_{2i-1})$. Notice that for each $i$, the map $\psi_i(u) := (u - y_{2i-1})^2$ is $L$-Lipschitz on $[-\Lambda, \Lambda]$ with $L := 2 (\Lambda + \Upsilon_Y)$ as we have $\abs{\psi_i'(u)} = 2|u - y_{2i-1}| \leq 2 (\Lambda + \Upsilon_Y)$. By Talagrand's lemma (see, for example, Lemma~5.7 of~\cite{mohri2018foundations}), we have
\begin{equation*}
\begin{aligned}
    &\Reb_{N/2}(\gC) \leq 2 (\Lambda + \Upsilon_Y)\; \Reb_{N/2}(\gHkp(\Lambda)),\\
    \text{where} \;\; &\Reb_{N/2}(\gHkp(\Lambda)) :=\E_{\epsilon}\cro{\sup_{h \in \gHkp(\Lambda)} \; \frac{2}{\mu} \; \bigg| \sum_{i=1}^\mu \epsilon_i \, h(\rho_{2i-1}) \bigg|}.
\end{aligned}
\end{equation*}
$\Reb_{N/2}(\gHkp(\Lambda))$ is the empirical Rademacher complexity of the RKHS ball $\gHkp(\Lambda)$ which, by the reproducing property of the kernel $\kappa$ together with Cauchy--Schwarz, satisfies
\begin{equation*}
\begin{aligned}
    \Reb_{N/2}(\gHkp(\Lambda)) &\leq \frac{2 \Lambda}{\mu} \E_{\epsilon}\cro{\normkp{\sum_{i=1}^\mu \epsilon_i \kappa(\rho_{2i-1}, \cdot)}}\\
    & \leq \frac{2 \Lambda}{\mu} \sqrt{\E_\epsilon \normkp{\sum_{i=1}^\mu \epsilon_i \kappa(\rho_{2i-1}, \cdot)}^2}\\ 
    &= \frac{2 \Lambda}{\mu} \sqrt{\E_{\epsilon} \cro{\vepsilon^\top K \vepsilon}} = \frac{2 \Lambda}{\mu}\sqrt{\Trc K},
\end{aligned}
\end{equation*}
where $\vepsilon := (\epsilon_1, ..., \epsilon_\mu)^\top \in \{\pm 1\}^\mu$, $K$ is the Gram matrix $$K =(\kappa(\rho_{2i-1}, \rho_{2j-1}))_{i,j},$$ and the equality $\E_{\vepsilon}[\vepsilon^\top K \vepsilon] = \Trc K$ follows from independence and $\E[\epsilon_i\epsilon_j] = \delta_{ij}$. Again, as $\kappa$ is norm-diagonal, we have $\Trc K \leq \mu$, hence $\Reb_{N/2}(\gHkp(\Lambda)) \leq \tfrac{2\Lambda}{\sqrt{\mu}}$, which yields in turn
\begin{equation*}
    \Reb_{N/2}(\gC) \leq  \; \frac{4 \Lambda (\Lambda + \Upsilon_Y)}{\sqrt{\mu}}.
\end{equation*}

\end{proof}

\subsubsection{Bounding the deviation of the window risk from the statistical one}
Before stating a bound for such a deviation, we first show a geometric decaying of the dependence of the reservoir output on the far past, which is related to the exponentially fading memory property of our reservoir.
\begin{proposition}[Geometric decay of dependence of reservoir outputs on inputs]\label{prop:geometric-decay}
    Let $\rvx = (x_t : t\in \sZ_-) \in \gIZ$ be a time series, and let $\rvx^w : = (x_{-w + 1}, ..., x_0) \in (\gI)^w$ be its last truncated window. Let $H^T_h$ be a \textsc{QuaRK} reservoir as described in Section~\ref{sec:methodology} where $H^T$ is composed of $R \geq 1$ sub-reservoirs each of contraction factor $\lambda_r$, $\gO$ are the $k$-local observables measured per sub-reservoir, and $h \in \gHkp(\Lambda)$, where $\kappa$ is a Mat\'ern-based kernel as defined in~\cref{eq:matern-kernel} parameterized on $\nu > 1$, and $\xi > 0$. Then we have
    \begin{equation}
        \abs{H^T_h(\rvx) - H^T_h(\rvx^w)} \leq \frac{2 \Lambda}{\xi} \sqrt{\frac{\nu R |\gO|}{\nu - 1}} \lambda_\star^w,
    \end{equation}
    where $\lambda_\star := \max_r \lambda_r,$ and the reservoir is initialized for the finite window case to an arbitrary state $\rho'_w \in \gSH^R$.
\end{proposition}
\begin{proof}
    Let the sequence of states produced by the reservoir via consuming the full time series and the truncated window be respectively \begin{equation*} \begin{aligned} U^T(\rvx) &:= (\rho_t \in \gSH^R : t \in \sZ_-), \\ U^T(\rvx^w) &:= (\rho'_t \in \gSH^R : t \in \{-w+1, ..., 0\}), \end{aligned} \end{equation*} where, by definition, the truncated evolution $(U^T(x^w))_{-w}$ is initialized at some state $\rho'_{-w} \in \gSH^R$, and both $U^T(x)$ and $U^T(x^w)$ are then driven by the same input sequence over the last $w$ time steps.
    
    Since the composed reservoir $T(\rho, x)$ is $\lambda_\star-$contractive with respect to. the trace norm $\normTr{\cdot}$, with $\lambda^\star := \max_r \lambda_r$, we have 
    \begin{equation}\label{eq:back-recursion}
    \begin{aligned} \normTr{\rho_0 - \rho'_0} &\leq \lambda_\star \normTr{\rho_{-1} - \rho'_{-1}}\\ &\leq \lambda_\star^2 \normTr{\rho_{-2} - \rho'_{-2}}\\ &\;\;\vdots\\ &\leq \lambda_\star^w \normTr{\rho_{-w} - \rho'_{-w}} \leq  2 \lambda_\star^w, \end{aligned}
    \end{equation}
    where the last inequality leverages the fact that the diameter of the space of density operators is $2$ with respect to. the trace norm.

   We now account for the readout. By definition, we have
$$\abs{H^T_h(\rvx) - H^T_h(\rvx^w)}  = \abs{h(\rho_0) - h(\rho'_0)}.$$
Using the reproducing property in $\gHkp$ we get
\begin{equation}\label{eq:reprod-rho-0}
\begin{aligned}
    \abs{h(\rho_0) - h(\rho_0')} &= \abs{\Braket{h, \kappa(\rho_0, \cdot) - \kappa(\rho_0', \cdot)}_{\scriptscriptstyle \kappa}}\\
    &\leq \normkp{h} \normkp{\kappa(\rho_0, \cdot) - \kappa(\rho_0', \cdot)},
\end{aligned}
\end{equation}
with $\Braket{\cdot, \cdot}_{\scriptscriptstyle \kappa}$ being the RKHS inner product. Furthermore, we have
\begin{equation}\label{eq:kernel-to-matern-profile}
\begin{aligned}
    \normkp{\kappa(\rho_0, \cdot) - \kappa(\rho_0', \cdot)}^2 &= \kappa(\rho_0,\rho_0) + \kappa(\rho'_0,\rho'_0) - 2 \kappa(\rho_0,\rho'_0)\\
    &= 2 \pth{\varphi(0) - \varphi\pth{\normeuc{m(\rho_0) - m(\rho'_0)}}},
\end{aligned}
\end{equation}
where~\cref{eq:matern-kernel} was used. Define the constant
$$
L_\kappa := \sup_{s > 0} \frac{\sqrt{2 \pth{\varphi(0) - \varphi(s)}}}{s}.
$$
We show in the next paragraph that $L_\kappa \leq \sqrt{-\varphi''(0)},$ where $\varphi''(0)$ is the second derivative of $\varphi$ at the origin, which exists for the Mat\'ern profile $\varphi$ as $\nu > 1$ (i.e., $\varphi \in C^2$ at $0$).

To show that $L_\kappa \leq \sqrt{-\varphi''(0)}$, we use the standard Bochner~\cite{Hofmann_Schölkopf_Smola_2008} result stating that for stationary covariance (kernel) functions such as $\psi(\eta) := \varphi(\normeuc{\eta})$, there exists a finite nonnegative symmetric spectral measure $\gamma$ such that 
\begin{equation*}
    \psi(\eta) = \int_{\sR^d} e^{i \omega^\top \eta} d\gamma(\omega) = \int_{\sR^d} \cos(\omega^\top \eta) d\gamma(\omega),
\end{equation*}
where the last equality holds as $\psi$ is real and even. Hence we have
\begin{equation*}
    \varphi(0) - \varphi(\normeuc{\eta}) = \psi(0) - \psi(\eta) = \int_{\sR^d} \pth{1 - \cos(\omega^\top \eta)} d\gamma(\omega).
\end{equation*}
Using the elementary inequality $1- \cos t \leq t^2/2$ (valid for all $t$) yields
\begin{equation}\label{eq:inequality-matern-profile}
\begin{aligned}
    \varphi(0) - \varphi(\normeuc{\eta}) &\leq \frac{1}{2} \int_{\sR^d} (\omega^\top \eta)^2 d\gamma(\omega)\\
    &\leq \frac{\normeuc{\eta}^2}{2} \int_{\sR^d} (\omega^\top e)^2 d\gamma(\omega), \quad e := \tfrac{\eta}{\normeuc{\eta}}.
\end{aligned}
\end{equation}
Consider the 1D restriction $g(t) = \psi(t e) = \varphi(\abs{t})$. Differentiating the spectral representation twice at point $t = 0$ (legitimate as $\varphi''(0)$ exists), we get
\begin{equation*}
    g''(0) = \frac{d^2}{dt^2} \int \cos(t \omega^\top e) d\gamma(\omega) \bigg|_{t=0} = - \int (\omega^\top e)^2 d\gamma(\omega) = - \varphi''(0),
\end{equation*}
where the last equality holds as $g''(0) = \varphi''(0)$. The last inequality of~\cref{eq:inequality-matern-profile} therefore yields $2 (\varphi(0)-\varphi(s)) \leq -\varphi''(0) s^2$, which directly leads to
\begin{equation}\label{eq:L-kappa-bound}
    L_\kappa  \leq \sqrt{-\varphi''(0)}.
\end{equation}
Hence setting $s = \normeuc{m(\rho_0) - m(\rho'_0)}$ yields
\begin{equation*}
\begin{aligned}
    \sqrt{2 \pth{\varphi(0) - \varphi\pth{\normeuc{m(\rho_0) - m(\rho'_0)}}}} &\leq L_\kappa \normeuc{m(\rho_0) - m(\rho'_0)} \\
    &\leq \normeuc{m(\rho_0) - m(\rho'_0)} \sqrt{-\varphi''(0)}.
\end{aligned}
\end{equation*}

Furthermore, let $z := \sqrt{2\nu}\,s/\xi$ so that the Mat\'ern profile (defined in~\cref{eq:matern-profile}) becomes
\[
\varphi(s)=\frac{2^{1-\nu}}{\Gamma(\nu)}\, z^\nu K_\nu(z).
\]
For $\nu>1$, the small-$z$ expansion of $K_\nu$ gives
\[
z^\nu K_\nu(z)=2^{\nu-1}\Gamma(\nu)\left(1+\frac{z^2}{4(1-\nu)}+\mathcal{O}(z^4)\right)\qquad(z\to 0).
\]
Substituting into the definition of $\varphi$ yields
\[
\varphi(s)=1-\frac{1}{4(\nu-1)}z^2+\mathcal{O}(z^4)
=1-\frac{\nu}{2(\nu-1)\xi^2}s^2+\mathcal{O}(s^4).
\]
Hence $\varphi''(0)=-\frac{\nu}{(\nu-1)\xi^2}$, i.e.,
\[
-\varphi''(0)=\frac{1}{\xi^2}\frac{\nu}{\nu-1}.
\]
Therefore, from~\cref{eq:L-kappa-bound} we get $L_\kappa \leq \tfrac{1}{\xi} \sqrt{\tfrac{\nu}{\nu - 1}}$, and hence,
combining~\cref{eq:reprod-rho-0} and~\cref{eq:kernel-to-matern-profile}, we get
\begin{equation}\label{eq:readout-to-moments-bound}
\begin{aligned}
    \abs{h(\rho_0) - h(\rho_0')} &\leq \normkp{h} \normkp{\kappa(\rho_0, \cdot) - \kappa(\rho_0', \cdot)}\\
    &\leq \normkp{h} L_\kappa \normeuc{m(\rho_0) - m(\rho'_0)}\\
    &\leq \normkp{h} \frac{1}{\xi} \sqrt{\frac{\nu}{\nu - 1}} \normeuc{m(\rho_0) - m(\rho'_0)}.
\end{aligned}       
\end{equation}

Finally, for the moment map $m(\rho) = \pth{\Tr[O_\ell \rho]}_{\ell = 1}^{R |\gO|}$, where $R |\gO|$ is the number of observables measured across all sub-reservoirs, and where each $O_\ell$ is formed from a $k$-local observable (from $\gO$) padded by identities to act on the full space $\gSH^R$ (this does not change the operator norm), we have for each $\ell$ by H\"older's inequality for Schatten norms
\begin{equation}
    \abs{\Tr\cro{O_\ell (\rho_0 - \rho'_0)}} \leq \norm{O_\ell}_{\infty} \normTr{\rho_0 - \rho'_0}.
\end{equation}
As the operator norm of a Pauli string satisfies $\norm{O_\ell}_\infty = 1$, we get
\begin{equation}\label{eq:moments-to-states}
    \normeuc{m(\rho_0) - m(\rho'_0)} \!=\! \pth{\sum_{\ell=1}^{R |\gO|} \abs{\Tr\cro{O_\ell (\rho_0 - \rho'_0)}}^2 }^{1/2} \!\!\!\! \leq \sqrt{R |\gO|} \normTr{\rho_0 - \rho'_0}.
\end{equation}
Combining Eqs.~(\ref{eq:back-recursion}),~(\ref{eq:readout-to-moments-bound}), and~(\ref{eq:moments-to-states}), along with the fact that $\normkp{h} \leq \Lambda$, we get the claimed bound
\begin{equation*}
    \abs{h(\rho_0) - h(\rho_0')} \leq \frac{\Lambda}{\xi} \sqrt{\frac{\nu}{\nu - 1}} \cdot \sqrt{R |\gO|} \normTr{\rho_0 - \rho'_0} \leq \frac{2 \Lambda}{\xi} \sqrt{\frac{\nu R |\gO|}{\nu - 1}} \lambda_\star^w.
\end{equation*}
\end{proof}

We now provide the bound on the deviation $\abs{R(H^T_h) - R^w(H^T_h)}$.
\begin{corollary}[Deviation of the window risk from the statistical one]\label{cor:deviation-window-risk-true}
    Let the conditions of Proposition~\ref{prop:geometric-decay} apply, with the additional assumption that the considered time series is drawn from the stationary process $\rmIO = ((X_t,Y_t): t \in \sZ_-)$ and that the outputs satisfy $|Y_t| \leq \Upsilon_Y$. Then we have
    \begin{equation}
        \abs{R(H^T_h) - R^w(H^T_h)} \leq  \frac{4 \Lambda (\Lambda + \Upsilon_Y)}{\xi} \sqrt{\frac{\nu R |\gO|}{\nu - 1}} \lambda_\star^w.
    \end{equation}
\end{corollary}
\begin{proof}
    We have
    \begin{equation*}
        R(H^T_h) - R^w(H^T_h) = \E\cro{\ell\pth{H^T_h(\rmX), Y_0}} - \E\cro{\ell\pth{H^T_h(\rmX^w_0), Y_0}},
    \end{equation*}
    where by stationarity we write $\rmX = (\ldots, X_{-w-1}, X_{-w}, \rmX^w_0)$. Hence we get
    \begin{equation*}
    \begin{aligned}
        \abs{R(H^T_h) - R^w(H^T_h)} &= \abs{\E\cro{\ell\pth{H^T_h(\rmX), Y_0} - \ell\pth{H^T_h(\rmX^w_0), Y_0}}}\\
        & \leq \E\cro{\abs{\ell\pth{H^T_h(\rmX), Y_0} - \ell\pth{H^T_h(\rmX^w_0), Y_0}}}\\
        &\leq \E\cro{L \abs{H^T_h(\rmX) - H^T_h(\rmX^w_0)}}\\
        &\leq L \frac{2 \Lambda}{\xi} \sqrt{\frac{\nu R |\gO|}{\nu - 1}} \lambda_\star^w,
    \end{aligned}
    \end{equation*}
    where the last inequality uses that the squared loss is $L$-Lipschitz on $[-\Lambda,\Lambda]\times[-\Upsilon_Y,\Upsilon_Y]$, with $L := 2 (\Lambda + \Upsilon_Y)$ (see the third paragraph of the proof of Corollary~\ref{cor:mse_rkhs_beta}), together with Proposition~\ref{prop:geometric-decay}.
\end{proof}

We finally are now ready to provide proof of Theorem~\ref{thm:generalization}.

\subsubsection{Proof of Theorem~\ref{thm:generalization}}\label{subsubsec:proof-thm-generalization}
Since conditions of Theorem~\ref{thm:generalization} satisfy both conditions of Corollary~\ref{cor:mse_rkhs_beta} and Corollary~\ref{cor:deviation-window-risk-true}, and from 
$$\abs{R(H^T_h) - R_N(H^T_h)} \leq \abs{R(H^T_h) - R^w(H^T_h)} + \abs{R^w(H^T_h) - R_N(H^T_h)},$$
we conclude that 
\begin{equation*}
    \abs{R(H^T_h) - R_N(H^T_h)} \leq \Reb_{N/2}(\ell \circ \gHkp(\Lambda)) + M(N,w,g, \delta') + G(h, \Lambda, T)
\end{equation*}
where $\Reb_{N/2}(\ell \circ \gHkp(\Lambda))$ denotes the empirical Rademacher complexity of the loss-composed class $\ell \circ \gHkp(\Lambda)$ given by
$$\Reb_{N/2}(\ell \circ \gHkp(\Lambda)) := 4\sqrt{2}\,\frac{\Lambda(\Lambda+\Upsilon_Y)}{\sqrt{N}},$$
 the mixing penalty $M(N,w,g)$ given by
$$ M(N, w, g) := 3(\Lambda+\Upsilon_Y)^2\sqrt{\frac{\log(4/\delta')}{N}}$$
and the window truncation geometrically decaying penalty given by
$$\mathfrak{G}(\kappa, \Lambda, T) := \frac{4 \Lambda (\Lambda + \Upsilon_Y)}{\xi} \sqrt{\frac{\nu R |\gO|}{\nu - 1}} \lambda_\star^w.$$
\qed

%% file: appendices/additional_exps.tex
\section{Additional details for numerical validation}
In this appendix, we provide further details on the data generation process and target functionals used in the numerical experiments.
\subsection{Data generation.}
\subsubsection{Input VARMA process}
\label{subsec:VARMA}
\paragraph{Definition.}
Let $\rmX := (X_t : t \in \Zminus)$ be a (strictly) stationary $\mathrm{VARMA}(p,q)$ process (sometimes called VARMA for vector VARMA), i.e.,
\begin{equation}\label{eq:arma-def}
    X_t \;=\; \sum_{i=1}^p \Phi_i X_{t-i} \;+\; \Theta_0 \, \epsilon_t \;+\; \sum_{j=1}^q \Theta_j \epsilon_{t-j},
\end{equation}
where $(\epsilon_t)_{t\in\Zminus}$ are i.i.d.\ centred innovations with $\E[\normeuc{\epsilon_t}^2]<\infty$, and $\Phi_i, \Theta_j$ are $d \times d$ matrices, with $\Theta_0$ non singular. Define the lag operator $L X_t := X_{t-1}$, and set the matrix polynomials
\[
\Phi(L) := I_d - \sum_{i=1}^p \Phi_i L^i, \qquad \Theta(L) := \Theta_0 + \sum_{j=1}^q \Theta_j L^j,
\]
then the model in~\cref{eq:arma-def} is $\Phi(L) X_t = \Theta(L) \epsilon_t$. It is shown in~\cite{mokkadem1988mixing} that if the innovations $(\epsilon_t)_{t\in\Zminus}$ are i.i.d.\ with a distribution
absolutely continuous with respect to the Lebesgue measure on $\R^d$ (i.e.\ admitting a density), and if the AR polynomial is stable
(i.e.\ the VARMA is causal), namely
\begin{equation}\label{eq:arma-stability}
\det\!\Big(I_d-\sum_{i=1}^p \Phi_i z^i\Big)\neq 0,\qquad \forall z\in\sC,\ |z|\le 1,
\end{equation}
together with the non-degeneracy condition of $\det(\Theta_0)\neq 0$,
then the strictly stationary unique solution of~\cref{eq:arma-def} is (in fact geometrically) $\beta$-mixing.

\paragraph{Experimental setting (stable VARMA$(p,q)$ family).}
Fix $(d,p,q)\in\sN^\ast\times\sN^\ast\times\sN$ and choose a stability budget $\gamma\in(0,1)$
(which controls the dependence strength: larger $\gamma$ yields slower mixing).
We generate a random stable VAR part as follows.
Draw i.i.d.\ matrices $(\widetilde\Phi_i)_{i=1}^p$ with entries $\sim \gN(0,1)$ and define the normalized directions
$U_i := \widetilde\Phi_i/\|\widetilde\Phi_i\|_2$.
Draw i.i.d.\ positive weights $w_i\sim\Unif(0,1)$ and set
$a_i := \gamma\, w_i/\sum_{k=1}^p w_k$, so that $\sum_{i=1}^p a_i=\gamma$.
Finally set
\[
\Phi_i := a_i\, U_i,\qquad i=1,\dots,p,
\]
which ensures $\sum_{i=1}^p\|\Phi_i\|_2=\gamma<1$ and therefore the AR-stability (causality) condition
$$\det(I_d-\sum_{i=1}^p \Phi_i z^i)\neq 0, \quad \text{for all} \; |z|\le 1.$$

For the MA part, we take $\Theta_0 := I_d$ (hence nonsingular) and generate
$\Theta_j := b_j V_j$ for $j=1,\dots,q$, where $V_j:=\widetilde\Theta_j/\|\widetilde\Theta_j\|_2$ with
$\widetilde\Theta_j$ i.i.d.\ Gaussian matrices and where $(b_j)_{j=1}^q$ are decaying amplitudes,
e.g.\ $b_j := \eta\,\rho^{\,j-1}$ with $\eta>0$ and $\rho\in(0,1)$.
The innovations $(\epsilon_t)_{t\in\Zminus}$ are i.i.d.\ with a Lebesgue density on $\R^d$,
for instance $\epsilon_t\sim\gN(0,\sigma^2 I_d)$.

We simulate the (unbounded) VARMA recursion
\[
Z_t=\sum_{i=1}^p\Phi_i Z_{t-i}+\Theta_0\epsilon_t+\sum_{j=1}^q \Theta_j \epsilon_{t-j}
\]
with a burn-in $B\gg 1$ and retain the last $T$ samples.
To match the bounded input set $\gI=(-1,1)^d$ used in the main experiments, we finally set
\[
X_t := \tanh(Z_t)\in(-1,1)^d \quad \text{(componentwise)}.
\]
Since $X_t$ is a measurable transformation of $Z_t$, it inherits the (geometric) $\beta$-mixing property of the VARMA process.

\subsection{Output processes (functionals)}
To obtain scalar labels $Y_t\in\R$ from vector inputs $X_t\in(-1,1)^d$, we consider three real-valued
fading-memory functionals $H^\star:(\gI)^{\sZ_-}\to\R$ evaluated on the past orbit
$\boldsymbol{X}_t := (X_{t-k})_{k\ge 0}$.
Throughout, we fix a window size $w\in\sN^\ast$ and use the window-truncated evaluation
$H^\star_w(\boldsymbol{X}_t)$, obtained by restricting $k\in\{0,\dots,w-1\}$.

\noindent\textbf{Random projection vectors.}
We generate a random unit vector $u\in\R^d$ by drawing $g\sim\gN(0,I_d)$ and setting
\begin{equation}\label{eq:rand-unit-u}
    u := \frac{g}{\|g\|_2}.
\end{equation}
When needed, we generate $v\in\R^d$ independently the same way and (optionally) orthogonalize it via
\begin{equation}\label{eq:rand-unit-v}
    \tilde v := v - \langle v,u\rangle u,\qquad v := \frac{\tilde v}{\|\tilde v\|_2}.
\end{equation}
This choice ensures $u$ (and $v$) are uniformly distributed on the unit sphere $\mathbb{S}^{d-1}$.

\noindent\textbf{(F1) Scalar one-step forecasting.}
Fix $u$ as in~\cref{eq:rand-unit-u} and define the forecasting functional
\begin{equation}\label{eq:F1}
    H^\star_{\rm fore}(\boldsymbol{X}_t) := u^\top X_{t+1},
    \qquad Y_t := H^\star_{\rm fore}(\boldsymbol{X}_t).
\end{equation}

\noindent\textbf{(F2) Exponentially fading linear functional.}
Fix $u$ as in~\cref{eq:rand-unit-u} and a decay $\alpha\in(0,1)$. Define
\begin{equation}\label{eq:F2}
    H^\star_{\rm exp}(\boldsymbol{X}_t)
    := \sum_{k=0}^{\infty}\alpha^k\, u^\top X_{t-k},
    \qquad
    Y_t := H^\star_{{\rm exp},w}(\boldsymbol{X}_t)
    := \sum_{k=0}^{w-1}\alpha^k\, u^\top X_{t-k}.
\end{equation}

\noindent\textbf{(F3) Truncated Volterra fading-memory functional (order 2).}
Fix $u,v$ as in~\cref{eq:rand-unit-u}--\cref{eq:rand-unit-v} and a decay $\alpha\in(0,1)$.
We define a (causal) Volterra functional of order $2$ with exponentially decaying kernels,
\begin{align}
    H^\star_{\rm vol}(\boldsymbol{X}_t)
    &:= \sum_{k=0}^{\infty} h_1(k)^\top X_{t-k}
      \;+\; \sum_{k=0}^{\infty}\sum_{\ell=0}^{\infty} X_{t-k}^\top H_2(k,\ell)\, X_{t-\ell},
      \label{eq:volterra-general}
\end{align}
where we take the rank-one fading kernels
\begin{equation}\label{eq:volterra-kernels}
    h_1(k) := \alpha^k u,
    \qquad
    H_2(k,\ell) := \frac{1}{2}\alpha^{k+\ell}\, v v^\top,
\end{equation}
so that
\begin{equation}\label{eq:volterra-expanded}
    H^\star_{\rm vol}(\boldsymbol{X}_t)
    = \sum_{k=0}^{\infty}\alpha^k\, u^\top X_{t-k}
    \;+\; \frac{1}{2}\sum_{k=0}^{\infty}\sum_{\ell=0}^{\infty}\alpha^{k+\ell}\,
    (v^\top X_{t-k})(v^\top X_{t-\ell}).
\end{equation}
In experiments, we use the window-truncated version
\begin{equation}\label{eq:F3}
\begin{aligned}
    Y_t := H^\star_{{\rm vol},w}(\boldsymbol{X}_t)
    := &\sum_{k=0}^{w-1}\alpha^k\, u^\top X_{t-k}\\
    &+ \frac{1}{2}\sum_{k=0}^{w-1}\sum_{\ell=0}^{w-1}\alpha^{k+\ell}\,
    (v^\top X_{t-k})(v^\top X_{t-\ell}).
\end{aligned}
\end{equation}
Since $X_t\in(-1,1)^d$ and $\alpha\in(0,1)$, all targets are bounded.

\subsection{Kernel hyper-parameters tuning}\label{app:kernel-tuning}

For each task, we select the Matérn kernel hyper-parameters once using a lightweight
train/validation procedure, and then keep them fixed across all reported runs for that task. The tuner is intentionally simple (single split and basic optimizers) to keep the overhead minimal compared to the cost of quantum feature.

\paragraph{Setup}
Given a dataset of quantum features and targets $\{(z_i,y_i)\}_{i=1}^N$ with
$Z \in \mathbb{R}^{N\times D}$ and $y\in\mathbb{R}^N$, we create a validation split by shuffling indices
with a fixed seed and allocating a fraction $\texttt{val\_ratio}=0.2$ to validation.
During hyper-parameter tuning, we fix a small ridge term $\lambda_{\rm reg}=10^{-6}$ to stabilize
the Gram-matrix inversion while focusing the search on the kernel shape.

\paragraph{Objective (validation MSE)}
For candidate Matérn hyper-parameters $(\xi,\nu)$ (length-scale and smoothness), we evaluate
the validation mean-squared error of kernel ridge regression (KRR). Concretely, we compute
\[
\widehat{y}_{\rm val}
\;=\;
K_{\rm val, tr}\,\alpha,
\qquad
\alpha \;=\; (K_{\rm tr,tr} + \lambda_{\rm reg} I)^{-1}y_{\rm tr},
\]
where $K_{\rm tr,tr} = [\kappa_{\xi,\nu}(z_i,z_j)]_{i,j\in \text{tr}}$ and
$K_{\rm val,tr} = [\kappa_{\xi,\nu}(z_i,z_j)]_{i\in\text{val},\,j\in\text{tr}}$.
The objective is $\mathrm{MSE}(y_{\rm val},\widehat{y}_{\rm val})$.

\paragraph{Grid over $\nu$ and bounded search over $\xi$}
We adopt the \texttt{grid} tuning strategy: we scan a small candidate set of smoothness values
\[
\nu \in \{0.5,\;1.5,\;2.5,\;5.0\},
\]
and, for each fixed $\nu$, we solve the one-dimensional problem
\[
\xi^\star(\nu)\; \in\; \arg\min_{\xi\in[10^{-3},\,10^{3}]}\; \mathrm{MSE}_{\rm val}(\xi,\nu),
\]
using bounded optimization in $\log(\xi)$ (maximum \texttt{xi\_maxiter}=80 iterations).
We then select the final pair as
\[
(\xi^\star,\nu^\star)\;\in\;\arg\min_{(\xi^\star(\nu),\,\nu)}\;\mathrm{MSE}_{\rm val}(\xi^\star(\nu),\nu),
\]
and store $(\xi^\star,\nu^\star)$ as the task-level Matérn kernel configuration.